\documentclass{article}
\def\withcolors{0}
\def\withnotes{0}
\def\doublecolumn{0}
\usepackage{ccanonne}
\usepackage{multicol}
\usepackage{multirow}
\usepackage[numbers]{natbib}

\newcommand{\perturb}{\gamma}

\newcommand{\zdims}{k}
\newcommand{\nsamps}{m}

\newcommand{\nspu}{m}
\newcommand{\nin}{r}

\newcommand{\samp}[1]{{X_{#1}}}
\newcommand{\sampi}[2]{{X_{#1, #2}}}
\newcommand{\samps}{{X^{\ns}}}
\newcommand{\msg}[1]{{Y_{#1}}}
\newcommand{\msgs}{{Y^{\ns}}}

\newcommand{\cA}{\mathcal{A}}

\newcommand{\x}{\mathbf{x}}

\newcommand{\risk}{\mathcal{R}}
\def\multiset#1#2{\ensuremath{\left(\kern-.3em\left(\genfrac{}{}{0pt}{}{#1}{#2}\right)\kern-.3em\right)}}

\newcommand{\hp}{\widehat{\p}}

\newcommand{\CI}{{C_I}}
\newcommand{\nloc}{{n_1}}
\newcommand{\CR}{{C_{R}}}

\newcommand{\bm}{\mathbf{m}}

\usepackage{hyperref}       
\usepackage{url}            

\title{Discrete Distribution Estimation under User-level\\ Local Differential Privacy}

\newcolumntype{C}{>{\centering\arraybackslash}p{0.27\textwidth}}
\begin{document}

\author{
\begin{tabular}{C C C}
  Jayadev Acharya & Yuhan Liu & Ziteng Sun\\
 Cornell University & Cornell University & Google Research\\ 
\small \texttt{acharya@cornell.edu} &\small \texttt{yl2976@cornell.edu} & \small \texttt{zitengsun@google.com}
\end{tabular}
}
\date{}

\maketitle

\begin{abstract}
    We study discrete distribution estimation under user-level local differential privacy (LDP). In user-level $\varepsilon$-LDP, each user has $m\ge1$ samples and the privacy of all $m$ samples must be preserved simultaneously. \new{We resolve the following dilemma: While on the one hand} having more samples per user should provide more information about the underlying distribution, on the other hand, guaranteeing privacy of all $m$ samples should make estimation task more difficult. \new{We} obtain tight bounds for this problem under almost all parameter regimes. Perhaps surprisingly, we show that in suitable parameter regimes, having $m$ samples per user is equivalent to having $m$ times more users, each with only one sample. Our results demonstrate interesting phase transitions for $m$ and the privacy parameter $\varepsilon$ in the estimation risk. \new{Finally, connecting with recent results on shuffled DP, we show that combined with random shuffling,
    our algorithm leads to optimal error guarantees (up to logarithmic factors) under the central model of user-level DP in certain parameter regimes. We provide several simulations to verify our theoretical findings.
    }
\end{abstract}

\section{Introduction}

Modern distributed machine learning systems such as federated learning~\cite{Kairouz2021open} collects data from users to provide better service. Without proper design, a learning algorithm \new{can} reveal sensitive information about the users. Differential privacy (DP)~\cite{DMNS:06}, which requires the algorithm's output to be ``similar'' when a single contribution changes,
has become the gold standard for privacy protection in many machine learning and database applications. 

In the distributed setting, the more stringent version, local differential privacy (LDP)~\cite{warner1965randomized, KLNRS:11} requires users to privatize their data before sending it to the data collector (see \cref{def:ldp}). {In other words, the true data never leaves the user.} \new{However, LDP comes at a significant drop in utility compared to central DP where a trusted central data collector performs differentially private analysis on user data.  To circumvent this,}  \new{a sequence of recent works~\cite{erlingsson2019amplification, 10.1007/978-3-030-17653-2_13, 10.1007/978-3-030-26951-7_22, girgis2021renyi, 9719772}, has shown that combined with random shuffling, locally randomized data can lead to an amplified DP guarantee in the central model. This setting is often referred to as the \textit{shuffle model of DP} and motivates more study of local randomizers with large privacy parameters. }

For the task of discrete distribution under LDP constraints, efficient algorithm and fundamental limits have been established in \cite{DJW:13:FOCS, EPK:14, KBR:16, YeB17, ASZ:18,chen2020breaking}. However, these works consider the setting where each user contributes a single data point. \new{The setting where multiple samples per user are allowed, which is common in practice, is largely unexplored.}


We study discrete distribution estimation \new{when each user has multiple data samples and must privatize all their samples under LDP. Notice that without privacy constraints, more samples per user means an increase in overall number of samples thus leading to a reduction in the estimation error. When each user has multiple samples, one can choose to ignore all but one sample from each user and obtain the same performance as the item-level LDP where user has one data sample, which will lead to the case performance as the one-sample case. When the users have multiple samples, we have hope of using information from these samples to obtain better estimators. However, the noise addition mechanism also becomes stringent because now changing the value of a data point means changing all the samples of a user.  We ask the following question.

\begin{center}
    \emph{Can multiple samples per user help with estimation while maintaining \\the same local privacy budget at each user?}
\end{center}

We settle this question affirmatively and \new{obtain a nearly tight characterization of the estimation error} for all values of $m$ (the number of samples per user) and $\eps$ (the privacy parameter). We show that in certain regimes, \emph{having $\nspu$ samples per user is equivalent to having $\nspu\ns$  users each with one sample and each with a privacy budget of $\eps$}.} Our results also demonstrate interesting phase transitions of the estimation risk in terms of privacy budget $\eps$ and number of samples per user $m$.

\new{Moreover, we show that combined with random shuffling, our results lead to optimal (up to logarithmic factors) estimation error in the central model of user-level privacy~\cite{liu2020user, esfandiari2021tight} in certain regimes of $\eps$ while maintaining the local privacy guarantee that the server only has access to a properly randomized version of user data. This also establishes the tight estimation error in the shuffle model of DP.}

\paragraph{Organization.} We \new{define the problem in} \cref{sec:setting} and state our results in \cref{sec:results}. We introduce our algorithms for the high privacy regime ($\eps<1$) in \cref{sec:high-priv}. Algorithms for the low privacy regime are discussed in \cref{sec:low-priv}. Finally we discuss lower bounds in \cref{sec:lower-bound}. Missing proofs are presented in the supplementary material.


\subsection{Problem setup and preliminaries} \label{sec:setting}
\new{Let $\Delta_\ab:=\{\p=(\p_1, \ldots, \p_\ab)\in \mathbb{R}^\ab: \p\ge 0, \normone{\p}=1\}$  be the $(\ab-1)$-dimensional probability simplex, which is the set of all $\ab$-ary distributions.} In this paper, we consider the homogeneous case where there are a total of $\ns$ users, each observing $\nspu$ i.i.d.\ samples from
the same (unknown) distribution  
$\p \in \Delta_\ab$. We write 
$\samp{i} := (\sampi{i}1, \sampi{i}2, \dots, \sampi{i}\nspu)$ for the samples at user $i$, and
$\samps = (\samp{1}, \samp{2}, \ldots, \samp{\ns}) \in [\ab]^{\ns\nspu}$ for all $\ns\nspu$ samples.

\paragraph{Remark on heterogeneity.} In practice, the data generation process can be heterogeneous. Our results can be extended to the case with limited heterogeneity on user distribution, \eg $\forall i \in [\ns], X_i \sim \p_i$ and $\totalvardist{\p_i}{\p} \le \gamma$.  Using the coupling argument in \cite[Appendix B]{NEURIPS2021_67e235e7}, the same results can be obtained when $\gamma$ is small (polynomial in $1/\nspu$ and $1/\ns$). We leave the study of the more general heterogeneous case as an interesting future work. 

To preserve privacy of users, (central) differential privacy requires an algorithm $\cA$ has ``similar" outputs on neighbouring datasets, formally defined below.
\begin{definition} \label{def:centraldp}
    An algorithm $\cA: [\ab]^{\ns \nspu} \rightarrow \cY$ is said to be $(\eps, \delta)$-DP at user level if for any $X^\ns$ and $X^{'\ns}$ which differ at one user's contribution, we have for any $S \subset \cY$,
    \[
        \bPr{\cA(X^\ns) \in S} \le e^\eps  \bPr{\cA(X^{'\ns}) \in S} + \delta.
    \]
\end{definition}

\new{The case of $\delta > 0$ is called approximate DP and $\delta = 0$ is pure DP, denoted as $\eps$-DP. When $\nspu = 1$, this is the same as item-level DP.}

\new{In the local model of DP, user $i$ sends a 
message $Y_i \in \mathcal{Y}$ to the central server through a channel $W_i$, which describes the randomized mapping from $[\ab]^\nspu$ to $\cY$. We require each $W_i$ to satisfy LDP constraints:}

\begin{definition} \label{def:ldp}
\new{A randomized scheme $W_i:[\ab]^m\to \cY$ satisfies $(\eps, \delta)$-LDP at user-level if $\forall \x, \x'\in[\ab]^m$ and $S \subset \cY$, 
\begin{equation} \label{eqn:def_ldp}
	W_i(y  \in S \mid \x) \le e^\eps\cdot W_i(y \in S \mid \x') + \delta.
\end{equation}}
\end{definition}

\new{For LDP, we will focus on the case when $\delta = 0$, denoted as $\eps$-LDP. All messaging schemes satisfying~\eqref{eqn:def_ldp} with $\delta = 0$ are denoted $\cW_\eps$. 
}

\new{Upon receiving $\msgs 
\eqdef (\msg{1}, \msg{2}, \ldots, 
\msg{\ns})$, the server outputs an estimator $\hp\colon \mathcal{Y}^\ns \to \Delta_{\ab}$ for the underlying distribution $\p$.} The performance of the estimator is measured by the expected total variation (TV) distance between $\hp$ and 
$\p$, where for $\p,\q\in \Delta_\ab$, $\totalvardist{\p}{\q} \eqdef (1/2) \sum_{x \in [\ab]} |\p(x) - \q(x)|$. In this work, we are interested in the \emph{minimax risk} of the estimation problem, defined as 
\begin{align}
	\risk(\eps, \ab, \ns, \nspu) \eqdef \min_{W^\ns} \min_{\hat{\p}} 
	\max_{\p \in 
	\Delta_{\ab}} \bEE{\totalvardist{\hp(\msgs)}{\p}}\,
\end{align}
where the minimum over $W^\ns$ is taken over all $\eps$-LDP messaging schemes.

In general, the choice of $W_i$ may 
depend on the previous messages $Y^{i-1} \eqdef (\msg{1}, \msg{2}, \ldots, \msg{i - 1})$ and a common random seed $U$ (independent of the observations) available to all users. A protocol is called \textbf{noninteractive} if all the channels $W_i$s are chosen independently of each other conditioned on the shared random seed. In distributed systems, noninteractive schemes are easier to implement and lead to lower latency. 

Next we introduce composition property of differential privacy and privacy amplification by shuffling, which we will use in later sections.
\begin{theorem}[(Advanced) composition \cite{DRV:10, DR:14}] If messaging schemes $W_1, W_2, \ldots, W_t$ satisfy $\eps$-LDP, then their composition $W^t = (W_1, W_2, \ldots, W_t)$ is $\eps'$-LDP with $\eps' = t\eps$ and $(\eps'', \delta)$-LDP with $\eps'' = \eps\sqrt{2t\log(1/\delta)} + t \eps (e^\eps - 1)$. Moreover, the choice of $W_i$ is allowed to depend on the outputs of $W_1, W_2, \ldots, W_{i-1}$.
\label{thm:composition}
\end{theorem}
\new{
\begin{theorem}[Amplification by shuffling~\cite{9719772}] Suppose messaging schemes $W_1, W_2, \ldots, W_\ns$ satisfy $\eps$-LDP. Let $\cA$ be the algorithm that applies $(W_1, W_2, \ldots, W_\ns)$ on $X^\ns_\pi = (X_{\pi(1)}, \ldots, X_{\pi(\ns)})$ where $\pi$ is a uniform premutation of $[\ns]$, then we have for $\delta \in (0, 1)$ satisfying $\eps \le \log(\frac{\ns}{16\log(1/\delta)})$, $\cA$ is $(\eps', \delta)$ - central DP for 
\[
    \eps' \le \log \Paren{ 1 + \frac{e^\eps - 1}{e^\eps + 1} \Paren{\frac{8 \sqrt{e^\eps \log(4/\delta )}}{\sqrt{n}} + \frac{8 e^\eps}{\ns}} }.
\]
\label{thm:shuffling}
\end{theorem}
When $\eps > 1$, we have $\eps' = O\Paren{\sqrt{\frac{e^\eps \log(1/\delta)}{n}}}$ and when $\eps \le 1$, $\eps' = O\Paren{\eps \sqrt{\frac{ \log(1/\delta)}{n}}}$. In the distributed setting, 
random shuffling is often performed by a secure multi-party communication protocol. Hence besides central DP guarantee, the model also guarantees that the server does not have access to the true user data. This model is often referred to as \emph{shuffle model}~\cite{erlingsson2019amplification, 10.1007/978-3-030-17653-2_13, 10.1007/978-3-030-26951-7_22, girgis2021renyi, 9719772}.
}

\section{\new{Prior work and our results}}\label{sec:results}
Distribution estimation under local privacy when each user has one sample ($\nspu = 1$) has been well-studied and it has been established that \cite{EPK:14, DJW:13:FOCS, KBR:16, YeB17, ASZ:18, AS:19},\footnote{We use $a \lor b = \max\{a, b\}$ and $a \land b = \min\{a, b\}$.}
\begin{align} \label{eqn:rate-one-sample}
\risk(\eps, \ab, \ns, \nspu = 1)=\Theta\left(\sqrt{\frac{\ab}{\ns}}\lor\sqrt{\frac{\ab^2}{n\newer{ \Paren{ (e^\eps - 1)^2  \land e^\eps}}}} \right).
\end{align}
\new{The first term is the centralized minimax risk without privacy constraints and the second term is the additional loss due to privacy. In our setup when each player has $\nspu$ samples, without privacy constraints when $\eps = \infty$, the server has unconstrained access to all $\ns \nspu$ samples giving a risk of}
\begin{equation} \label{eqn:stat_rate}
    \risk(\eps = \infty, \ab, \ns, \nspu )=\Theta\left(\sqrt{\frac{\ab}{\ns \nspu}}\right).
\end{equation}

\new{Therefore the first term of minimax risk reduces by a factor of  $1/\sqrt{\nspu}$ compared to the case when each user has one sample. The conundrum we try to resolve is about the second term. Can one take advantage of the multiple samples per user or does the requirement of guaranteeing privacy to all samples overwhelm the minimax risk?}

\smallskip
\new{Consider the case when $\eps = O(1)$. If we only use one sample from each user, we recover the rate $O\Paren{\sqrt{\ab^2/ (\ns\eps^2)}}$ for the case of $\nspu = 1$ under $\eps$-LDP. Another approach is to use a naive element-level LDP algorithm and apply composition of LDP (\cref{thm:composition}) to get user-level privacy guarantee. This leads to a rate of 
$O\Paren{\sqrt{\nspu \ab^2/ (\ns\eps^2)}}$ under pure LDP
or $\tilde{O}(\sqrt{\ab^2/ (\ns\eps^2)})$ if we relax to approximate LDP and use advanced composition. Either case, the risk does not decrease with $\nspu$.  The question of whether increasing $\nspu$ brings an advantage is still unclear.}

Another important question for the general $\nspu>1$ case is the dependence on the privacy parameter $\priv$. From ~\eqref{eqn:rate-one-sample}, when $\nspu=1$, the error rate decreases exponentially with respect to $\eps$ when $\eps \in (1, \ln\ab)$. With $\nspu>1$, can we still enjoy this exponential rate, and if so, for what ranges of $\priv$?

\newer{\paragraph{Our results.}} In this work, we \new{answer} these questions, showing that increasing $\nspu$ can \emph{indeed help} in certain regimes and the rate can be as steep as $O(1/\sqrt{\nspu})$ as in the centralized case. Moreover, we characterize the precise dependency on $\eps$, which has more sophisticated phase transitions compared to the case where $\nspu = 1$. Our results are summarized in Table~\ref{tab:summary}.

\begin{table*}[t]
        \centering
        \def\arraystretch{2.3}
        \begin{tabular}{|c|c|c|c|c|}
        \hline
        \multirow{2}{4em}{Regime}   & \multirow{2}{4em}{$\priv<1$} &\multicolumn{3}{|c|}{$\priv>1$}\\\cline{3-5}
        & & $\nspu<\ab/e^\priv$ & $\;\ab/e^\priv\le m<k\;$ & $\;m\ge k\;$\\
        \hline
        Upper bound & \multirow{2}{4em}{$\sqrt{\frac{k^2}{mn\priv^2}}$ (Thm \ref{thm:rate_small_eps})} & \multirow{2}{4em}{$\sqrt{\frac{k^2}{mne^\priv}}$ (Thm \ref{thm:small_m})}& $\!\! \sqrt{\frac{\ab}{\nspu\ns }} \!\! + \!\! \sqrt{\frac{k\log(\frac{\ab}{\nspu}+1)}{n\priv}}$ (Thm \ref{thm:mid_m}) & \multirow{2}{8em}{$\sqrt{\frac{\ab}{\nspu\ns }} +\sqrt{\frac{k^2}{mn\priv}} \dagger $  \\(Cor \ref{cor:large_m}, Thm~\ref{thm:lb-large-m}) }\\
        \cline{4-4}\cline{1-1}
        Lower bound & & &$\sqrt{\frac{\ab}{\nspu\ns }} + \sqrt{\frac{k}{n\priv}} \dagger $  
        (Thm~\ref{thm:lb-large-m})
        &\\\hline
        \end{tabular}
        \vspace{+5pt}
        \caption{Estimation risks for different parameter regimes of $\nspu$ and $\eps$ (omitting constants).
        The upper bounds hold with mild regularity (see theorem statements for details). For risks marked with $\dagger$, the lower bounds hold only when $\ns > (\ab/\eps)^2$.
        }
        \label{tab:summary}
    \end{table*}
    
For sufficiently large $\ns$, our rates are tight up to constant factors in all regimes except in $\ab/e^\priv\le\nspu<\ab$ where it is tight up to log factors. Somewhat surprisingly, for $\eps<1$ or $\nspu<\ab/e^\priv$, the error rate is the same as having $\nspu$ times more users in the one sample case,  but \newer{the sum of privacy budgets of all users} is $m$ times smaller. Next we look at $\nspu$ and $\eps$ separately and discuss their rates in different regimes.

\paragraph{Dependence on $\nspu$.} When $\eps<1$, the error rate always decays as $\Theta(1/\sqrt{\nspu})$. For $\priv\ge 1$, the error rate with respect to $\nspu$ differs for small $\nspu$ ($\nspu<\ab/e^\priv$), medium $\nspu$ ($\ab/e^\priv<\nspu<\ab$), 
and large $\nspu$ ($\nspu>k$). For small $\nspu$ and large $\nspu$, the error decays as $\sqrt{\nspu}$, but the dependence on $\priv$ is different. 
For medium $\nspu$, however, the error barely improves with $\nspu$ by at most a logarithmic factor.  
It is an interesting future direction to study
whether this logarithmic factor is tight.

\paragraph{Dependence on $\priv$.} In the high privacy regime ($\priv<1$), the error decays at a rate of $\Theta(1/\priv)$. The situation in the low privacy regime ($\eps>1$) is more complicated. When $\nspu < \ab$, we observe a phase transition at $\priv=\ln(\ab/\nspu)$. Below this threshold, there is an exponential decay with respect to $\priv$. Beyond $\ln(\ab/\nspu)$, the rate of decay becomes $\Theta(\sqrt{\eps})$.  If $\nspu>\ab$, then the exponential phase does not exist. When $\priv\ge\ab$, the error matches that of $\priv=\infty$ and cannot be improved further by increasing $\priv$.

\subsection{Connection to central and shuffled DP at user level} \label{sec:connection}

\new{
Our results imply almost tight rates in the central and shuffle model of DP under certain parameter regimes through  amplification by shuffling. In particular, we get the following result. 

\begin{theorem}\label{thm:shuffle}
    For $\nspu < \ab$ and $\eps$ and $\delta$ satisfying $\eps < \sqrt{\frac{\ab \log(1/\delta)^2}{\nspu \ns }}$ and $\delta \in (0,1/\ns)$, using algorithms in \cref{thm:rate_small_eps}
 and \cref{thm:small_m} combined with random shuffling, the estimation risk under $(\eps, \delta)$ user-level DP in the shuffle model is
 \[
    O\Paren{\frac{\ab \log(1/\delta)}{ \ns \sqrt{\nspu} \eps}}.
 \]
 \end{theorem}
 
Up to logarithmic factors, the bound matches the tight user-level central DP risk established in \cite{liu2020user, esfandiari2021tight}. which scales as $\tilde{O}(\sqrt{\ab/ (\ns \nspu)} + \ab/(\ns \sqrt{\nspu}\eps))$. Hence it is also tight up to logarithmic factors under shuffle DP. An interesting observation is that the privacy term for central/shuffle DP and local DP have differnt dependence on $\ns$.

We obtain the bound by applying amplification by shuffling (\cref{thm:shuffling}) to the LDP algorithm for $\eps < 1$ and $1 \le \eps \le \log (\ab/\nspu)$.
The above regime of $\eps$ covers both the $1/\eps$ decay rate when $\eps < 1$ and the $1/e^{\eps/2}$ decay rate when $\eps \ge 1$ and $\nspu < \ab/e^\eps$ in the local setting, showing the benefit of obtaining tight rates for large $\eps$ in LDP.
Whether this can be achieved for a wider range of $\eps$ is an interesting future question. We present the details in the supplementary.
}

\subsection{\new{Our approach}} 
\new{How to utilize the increased sample size at each user while preserving the same level of privacy is the central question to be resolved to design algorithms for $\nspu>1$.} A natural observation is that with $\nspu$ samples, each user can obtain a rough estimate of the entire distribution $\p$ with its local samples.
\begin{observation}
\new{For $x\in [\ab]$, let $Z_i(x)$ be the counts of $x$ in user $i$'s samples. Then the empirical frequency $Z_i(x)/\nspu$ is concentrated around $\p(x)$ with a standard deviation of $O(\sqrt{\p(x)(1 - \p(x))/\nspu})$.}
\label{obs:concentration}
\end{observation}

Our algorithm for $\priv\le 1$ relies on this observation. We provide a motivation for the special case of $\ab=2$, where we just need to estimate $p\eqdef\p(1)$. If $p$ is known to be in an interval $I$ of length $O(\sqrt{p(1-p)/\nspu})$, then the derivative of the following function is large is in $I$,
\[
P_{\nspu, t}(p)\eqdef\probaDistrOf{Z\sim \binomial{\nspu}{p}}{Z/\nspu>t}.
\]
To achieve the centralized rate, it suffices to send the indicator  $\indic{Z_i(1)/\nspu>t}$ where $t\in I$. The server then obtains an empirical estimate of $\probaOf{Z_i(1)/m>t}$, and evaluate the inverse function $P_{\nspu, t}^{-1}$ at the empirical estimate to obtain $\hat{p}$. To ensure privacy, the bits of users can be privatized using Randomized Response~\cite{warner1965randomized}. One remaining ingredient is how to obtain the interval $I$. \newer{For this part, we will rely on \cref{obs:concentration} and apply a private selection-type algorithm,} which we will elaborate in~\cref{sec:high-priv}.

For $\eps>1$, the situation becomes more complicated since we also want to enjoy the benefit of increased privacy budget, especially for $\nspu<\ab/e^\priv$ where an exponential decay in $\priv$ is expected. We observe another benefit of having more local samples.
\begin{observation}
 For any subset $S\subseteq[\ab]$, The probability that a user observes at least one sample in $S$ is $P_\nspu(S)= 1-(1-\p(S))^\nspu$.
 \label{obs:increased-prob}
\end{observation}
The idea is to divide the domain $[\ab]$ into $\nspu$ subsets of equal size, denoted by $B_1, \ldots, B_\nspu$. The users are also divided into $\nspu$ groups, each responsible for estimating the probability of symbols in just one \newer{subset}. A user can only send useful information about a \newer{subset} $B_j$ if it observes at least \new{one} sample \new{in} $B_j$. If $\nspu=1$, this happens with probability $\p(B_j)$. However, with $\nspu$ samples, the probability increases to $P_m(S)$. At least 90\% of the blocks satisfy $\p(B_j)\le 10/\nspu$, in which case $P_\nspu(B_j)=\Theta(\nspu\p(B_j))$. Hence, the number of effective messages sent by the users roughly increases by a factor of $\nspu$.

\paragraph{Connection to \cite{ACLST:21}.}
\cite{ACLST:21} studied a similar problem under communication constraints where each user sends a message of at most $\ell$ bits. They show that more samples per user decreases the error by $O(1/\sqrt{\nspu})$ in certain parameter regimes. While our algorithms are inspired by their algorithms, nontrivial extensions and novel ideas are needed to obtain tight rates in the LDP case. We highlight the important differences with \cite{ACLST:21} in terms of algorithm design and proof technique below.
\begin{enumerate}
    \item \emph{Localization stage.} In the localization stage, the analysis for the Gray code scheme in \cite{ACLST:21} fails \newer{since the bits are not private. This issue cannot be resolved by flipping the bits sent by the Gray code scheme using Randomized Response since it requires the error probability for most of the bits in the Gray code to decrease exponentially.} 

\newer{In this work, we view the localization localization stage as a private selection problem and resolve it based on private sparse distribution estimation in \cite{pmlr-v132-acharya21b}.} 
\newer{In addition to circumventing the failure issue mentioned above in the LDP case, this new idea can also be used in the communication-constrained case considered in \cite{ACLST:21}, which leads to
a simpler analysis and better regularity condition. For example, Theorem 2.1 of \cite{ACLST:21} requires $\ns/\log\ns=\Omega(\ab\log\nspu)$ for 1-bit algorithms, while using communication-limited sparse distribution estimation algorithm in \cite{pmlr-v132-acharya21b} only requires $\ns=\Omega(\ab\log\nspu)$}.

\item  A \emph{unified algorithm} for $\nspu\le \ab/e^{\priv}$ and $\ab/e^{\priv}\le \nspu\le \ab$. 
For the algorithms with $\nspu\le \ab/e^{\priv}$, we divide the domain $[\ab]$ into $\nspu$  bins instead of $\ab/e^{\priv}$ as suggested by \cite{ACLST:21}. 
Intuitively, this modification ensures that for a uniform $\p$, for any block $B_j$, $P_m(B_j)$ is some constant away from 0 and 1, which ensures that privatization does not lose too much information. 
Moreover, the algorithms for $\nspu\le \ab/e^{\priv}$ and $\ab/e^{\priv}\le \nspu\le \ab$ are now unified. We can make the same modification to the algorithms in \cite{ACLST:21} for $\nspu\le \ab/2^\ell$ and $\ab/2^\ell \le \nspu\le \ab$.

\item  \emph{Lower bound proof. }\cite{ACLST:21} uses a Poissonization trick, where each user needs to send one bit to indicate whether they get enough samples \newer{under Poisson sampling}, which might violate privacy constraints. 
We resolve this issue differently in different regimes. See \cref{sec:lower-bound} for a detailed discussion.
\end{enumerate}

\section{Algorithms for high privacy regime ($\eps \le 1$)}
\label{sec:high-priv}
We focus on the high privacy regime ($\eps \le 1$) and show that having more samples per user indeed brings an advantage and the rate decreases as $\Theta(1/\sqrt{\nspu})$.
\begin{theorem}
\label{thm:rate_small_eps}
When $\eps<1$ and $\ns\ge C \ab\log(\nspu)/\priv^2$ for some constant $C$,
\[
\risk(\eps, \ab, \ns, \nspu) = \Theta\Paren{ \sqrt{\frac{\ab^2}{\nspu\ns \eps^2}}}.
\]
Moreover, the bound is achieved by a \emph{non-interactive} protocol.
\end{theorem}

We describe the upper bound part in this section and discuss the lower bound idea in \cref{sec:lower-bound}. \newer{For simplicity, we describe the \emph{interactive} algorithm in this section, which carries most of the algorithmic ideas. We discuss how to modify the algorithm to a \emph{non-interactive} version in~\cref{app:binomial}.}

Inspired by \cite{ACLST:21}, we start with the special case of $\ab=2$ and then generalize to $\ab>2$. 

\subsection{Coin estimation ($k=2$)}
We first consider a simple coin estimation problem, which corresponds to the special case of $k=2$: There are $\ns$ users, each has $\nspu$ i.i.d. samples from $\bernoulli{p}$. The goal is to estimate $p$ under $\eps$-LDP. Our solution to this simple problem will become a crucial building block for algorithms in the general case.
The formal guarantee is stated below.
\begin{theorem}
\label{thm:binomial_ldp}
For $\priv<1$, there exists an algorithm with an estimate $\hat{p}$ such that if $\ns\ge C\log(\nspu)/\priv^2$  for some constant $C$, 
\[
\expect{(\hat{p}-p)^2}=O\left({1}/\Paren{\nspu\ns\priv^2}\right).
\]
\end{theorem}


Let $Z_u\sim\binomial{\nspu}{p}$ be the number of 1's in user $u$.
Our algorithm is inspired by \cite[Section 2.1]{ACLST:21} and consists of two stages. In the first stage (\textbf{localization}), we estimate $p$ up to accuracy $O(\sqrt{ p(1-p)/\nspu})$, the standard deviation of the local empirical estimate $Z_i/\nspu$. Then in the second stage (\textbf{refinement}), we try to obtain a more accurate estimate by inverting a binomial density function. 

Similar to \cite{ACLST:21}, we divide the $[0, 1]$ interval into $\Theta(\sqrt{\nspu})$ sub-intervals. \newer{At a high level, the intervals are designed such that if $p \in I_i$, there exists $c$, such that
\[
    (p - c\sqrt{\frac{p(1-p)}{\nspu}}, p + c\sqrt{\frac{p(1-p)}{\nspu}}) \subset I_{i-1} \cup I_i \cup I_{i + 1}.
\]
This is important for the localization stage since by \cref{obs:concentration}, we know that the empirical estimate of $p$ will lie in an interval close to $p$. 
} 
Let $\CI$ be a 
constant and $r\eqdef \lfloor\sqrt{\frac{\nspu}{2\CI}}\rfloor$. We define a partition $\{I_i\}_{i \in [2\nin]}$. 
Let
$I_i\eqdef[l_{i-1}, l_{i}]$ for $1\leq i\leq r$, where
\[
  l_i\eqdef \min\left\{\frac{\CI i^2}{\nspu}, \frac{1}{2}\right\}, \quad 0\le i\le r.
\]
Furthermore $I_{2r+1-i}\eqdef[1-l_i, 1-l_{i-1}]$. 

\newer{Next we describe the algorithm, we divide users into two groups $S_1, S_2$ with equal size, which will be used for the localization stage and refinement stage respectively.}
\newer{\paragraph{Localization stage.} 
In this stage, the server obtains a crude estimation of $p$ based on messages from $S_1$.

\begin{enumerate}
    \item \emph{Privatization scheme.} For $u \in S_1$, let $V_u$ be a $2r$-dimensional binary vector with $\forall i \in [2r], V_u(i) =  \indic{Z_u\in I_i}$, 
 which is a one-hot vector indicating the index of the interval that $Z_u$ falls in. Let $Y_u$ be obtained by flipping each coordinate of $V_u$ with probability $\beta := 1/(e^{\priv/2} + 1)$, \ie $\forall i \in [2r]$,
 \[
    Y_u(i) = \begin{cases}
        V_u(i) & \text{ with prob } 1 - \beta, \\
        1 - V_u(i) & \text{ with prob } \beta.
    \end{cases}
 \]
 \item \emph{Estimation scheme.} Here we obtain a confidence interval of $p$ using $Y_u$'s, whose index is given by
\[
\hat{i}=\arg\max_{i\in[2r]}\sum_{u\in S_1}Y_{u}(i).
\]
\end{enumerate}
}

\newer{\paragraph{Refinement stage.} In this stage, users in $S_2$ send messages based on $\hat{i}$ and the server obtains a refined estimate of $p$.
\begin{enumerate}
    \item \emph{Privatization scheme.} Let $t$ be the mid point of $I_{\hat{i}}$. Users in $S_2$ send a privatized version of $\indic{Z_u/\nspu>t}$, \ie
    \[
        Y_u = \begin{cases}
        \indic{Z_u/\nspu>t} & \text{ with prob } 1 - \beta, \\
        1 - \indic{Z_u/\nspu>t} & \text{ with prob } \beta.
        \end{cases}
    \]
    \item \emph{Estimation scheme.} Let $P_{\nspu, t}(p) \eqdef \bPr{Z_u/\nspu>t \mid Z_u \sim \binomial{n/2}{p}}$ and 
    \[
        \hat{P} \eqdef \frac{2}{\ns} \sum_{u \in S_2} Y_u,
    \]
    which is the empirical estimate of $P_{\nspu, t}(p)$. Return $\hat{p}=P_{\nspu, t}^{-1}(\hat{P})$.
\end{enumerate}
We defer the detailed analysis of the algorithm to~\cref{app:binomial}. In the localization stage, we show that $p  \in  I_{\hat{i}}\cup I_{\hat{i}-1}\cup I_{\hat{i}+1}$ with high probability. In the refinement stage, it is shown in~\cite{ACLST:21} that if the above holds, the gradient of $P_{\nspu, t}(p)=\bPr{Z_u/\nspu>t}$ with respect to $p$ is roughly $\Omega(\sqrt{\nspu})$. Hence, evaluating $\hat{p}=P_{\nspu, t}^{-1}(\hat{P})$ yields a squared error of $1/\nspu\ns\priv^2$ as desired. 
}
\subsection{General case $\ab>2$}

Using the algorithm for coin estimation, we can design an algorithm for $\ab>2$ using ideas from the 1-bit Hadamard Response algorithm~\cite{AS:19}.

Without loss of generality assume $\ab$ is a power of 2. Let $H_k$ be the Hadamard matrix defined as
\[
H_1=1, \;H_{2^l}=\begin{bmatrix}H_{2^{l-1}} & H_{2^{l-1}}\\
H_{2^{l-1}} & -H_{2^{l-1}}
\end{bmatrix},\forall l\ge 1.
\]
Let $T_i=\{j\in[\ab]:H_k(i, j)=1\}$ be the locations of 1's in the $i$th row of $H_\ab$. Users are divided into $k$ groups of  size $\ns/\ab$, each responsible for estimating one of $\p(T_i)$. By Theorem~\ref{thm:binomial_ldp}, we can obtain $\hp_T(i)$ such that 
\[
\expect{(\p(T_i)-\hp_T(i))^2}= O\left(\frac{k}{\nspu\ns\priv^2}\right).
\]
Let $\hp_T=(\hp_T(1), \ldots, \hp_T(\ab))$. We obtain $\hat{\p}$ with inverse Hadamard transform $\hp = H_k^{-1}(2\hp_T-\mathbf{1}_k)$.
Let $\p_T=(\p(T_1), \ldots, \p(T_\ab))$. Since $H_\ab^\top H_\ab=\ab I$, we have \[
\expect{\norm{\hp-\p}_{2}^2} =\frac{1}{\ab}\expect{\norm{\hp_T-\p_T}_{2}^2} = O\left(\frac{k}{\nspu\ns\priv^2}\right).
\] 
Applying Cauchy-Schwarz inequality, we can obtain the desired accuracy in Theorem~\ref{thm:rate_small_eps}.

\section{Algorithms for low privacy regime ($\priv>1 $)}
\label{sec:low-priv}
In this regime, the main challenge is to design algorithms that takes full advantage of both the increasing sample size $\nspu$ and extra privacy budget $\priv$. One may easily propose a simple extension of the algorithm for $\priv<1$: each user split the privacy budget into $\lfloor\priv\rfloor$ parts using the composition property of LDP (\cref{thm:composition}), each with a budget of 1 (the excess budget is omitted). Now each user can send information about $\lfloor\eps\rfloor$ different rows in $H_k$. Hence the effective sample size increases by a factor of $\eps$. Using Theorem~\ref{thm:rate_small_eps}, the guarantee of this algorithm is given by Corollary~\ref{cor:large_m}
\begin{corollary}
\label{cor:large_m}
For $\priv>1$, if $\ns>C\ab\log(\nspu)/\priv$ for some constant $C$, the simple extension outputs an estimate $\hp$ with
\[
\expect{\totalvardist{\hp}{\p}}=O\left(\sqrt{\frac{\ab^2}{\nspu\ns\priv}}\right).
\]
\end{corollary}

Hence we can easily achieve a risk with $1/\sqrt{\priv}$ decay. Can we acheive better rates? It turns out that when $\ns > (\ab/\eps)^2$, for large $\nspu$ ($\nspu>\ab$) the simple extension achieves the following optimal risk.
\begin{theorem}
\label{thm:lb-large-m}
For $\ns>(k/\eps)^2$, $\nspu\ge \ab$, and $\eps>1$, the minimax error rate satisfies
\[
\risk(\eps, \ab, \ns, \nspu)=
\Omega\Paren{\sqrt{\frac{\ab}{\nspu\ns}}\lor \sqrt{\frac{k^2}{\nspu\ns\priv}}}.
\]
\end{theorem}

For small $\nspu$ and medium $\nspu$, we can design better algorithms, which we will elaborate  in this section.

\subsection{Small $\nspu$ ( $\nspu\le \ab/e^\priv$)}
For small $\nspu$, we are able to obtain the same $\Theta(1/\sqrt{\nspu})$ decrease in the rate as in the high privacy case. Moreover, the error rate decays exponentially with $\priv$, similar to the error rate for $\nspu=1$.
\begin{theorem}
\label{thm:small_m}
When $\priv > 1$ and $\nspu < \ab/e^\eps$, if $\ns > C \nspu \log(\nspu)$, we have
\[
\risk(\eps, \ab, \ns, \nspu) = \Theta\Paren{ \sqrt{\frac{\ab^2}{\nspu\ns e^\eps}}}.
\]
The bound is achieved by a non-interactive protocol.
\end{theorem}

We focus on the upper bound part in this seciton and discuss the lower bound proof in \cref{sec:lower-bound}. At first glance, it may seem overly ambitious to achieve both exponential decay in $\priv$ and $1/\sqrt{\nspu}$ improvement in $\nspu$. Nevertheless, we accomplish this goal by taking advantage of both Observation~\ref{obs:concentration} and~\ref{obs:increased-prob}, and using the algorithm for $\nspu=1$ which enjoys exponential dependence on $\priv$ as a subroutine. Details of the algorithm are described as follows,

\begin{enumerate}
    \item Let $\priv_0=0.5$\footnote{We choose $\priv_0 = 0.5$ for simplicity. Any constant $\priv_0 < 0.5$ will work without changing the bounds up to constant.}. Divide the domain $[\ab]$ into $\nspu$ blocks $B_1, \ldots, B_\nspu$, each with size $
    \ab/\nspu$. 
    \item Each user uses $\priv_0=0.5$ to estimate the block distribution $\p_B\eqdef [\p(B_1), \ldots, \p(B_\nspu)]$ with the algorithm for $\priv\le 1$ in Section~\ref{sec:high-priv}. Denote the estimate as $\hp_B=[\hp_B(1), \ldots, \hp_B(\nspu)]$.
    \item Divide all users into $\nspu$ groups. The $j$th group tries to estimate $\bar{\p}_j:=\p(\cdot |B_j)$,
    the distribution conditioned on a sample is in $B_j$ (treated as uniform if $\p(B_j)=0$). Note that for $x \in B_j$,
    $
        \bar{\p}_j(x) = \frac{\p(x)}{\p(B_j)}.
    $
    
    To do this, each user in the $j$th group considers the distribution $\tilde{{\p}}_j$ over $B_j \cup \{\bot\}$ where
    \[
        \tp_j(\bot) := \bP{X^\nspu \sim \p}{ \forall x \in B_j, x \notin X^\nspu} = \Paren{1 - \p(B_j)}^\nspu,
    \]
    and for $x \in B_j$, $\tp_j(x)$ is the probability that $x$ is the first symbol in $B_j$ that appears in $X^\nspu$ . It can be obtained that
    \[
        \tp_j(x) = \bp_j(x) \Paren{1 - \tp(\bot)}.
    \]
    A user can simulate a sample from $\tp_j$ by getting $\bot$ if $B_j\cap X^\nspu=\varnothing$ and getting the first sample in $X^m \cap B_j$ if it is not empty. 
    Each user then sends a message using Hadamard Response~\citep{ASZ:18} for $(\eps - \eps_0)$-LDP.
    
    The server can then get an estimate $\hat{\tp}_j$ for $\tp_j$ using the messages above. Using $\hat{\tp}_j$, an estimate $\hp_j$ for $\bp_j$ can be obtained by  $\forall x \in B_j$
    \[
        \hp_j(x) =\frac{\hat{\tp}_j(x)}{1 - \hat{\tp}_j(\bot)},
    \]
    or $\nspu / \ab$ if $1 - \hat{\tp}_j(\bot) = 0$.
    \item To obtain an estimate $\hp$ for the underlying distribution, for each $x\in B_j$, 
    \[
    \hat{\p}(x)=\hat{\p}_B(j)\cdot \hat{\p}_j(x).
    \]
\end{enumerate}

To derive the guarantee for the algorithm, we need to relate the estimation errors for $\p, \tp_j$, and $\bp_j$. 
\newerest{
\begin{lemma}
\label{lem:decompose}
The estimation errors can be decomposed as
\begin{align}
    \expect{\totalvardist{\hp}{\p}}&\le\sum_{j \in [\nspu]} \frac{\p(B_j)}{(\nspu\p(B_j))\land 1}\expect{ \totalvardist{\hat{\tp}_j}{\tp_j} } \nonumber\\
    &\quad \quad  \quad \quad+\expect{\totalvardist{\hat{\p}_B}{\p_B}} \label{equ:decompose}
\end{align}
\end{lemma}
From~\cref{thm:rate_small_eps}, when $\ns/\nspu>C\log(\nspu)$,
\begin{equation}
    \expect{\totalvardist{\hp_B}{\p_B}}= O\Paren{ \sqrt{\frac{\nspu^2}{mn}}} = O\Paren{ \sqrt{\frac{\ab^2}{\nspu \ns e^\priv}}}.
    \label{equ:block_decomposition}
\end{equation}
The second inequality is due to $\nspu\le \ab^{2}/(\nspu e^{\priv})$ whenever $\nspu \le \ab/e^{\priv/2}$. By the guarantee of the Hadamard Response algorithm~\cite[Corollary 8]{ASZ:18}, 
$$\expect{\totalvardist{\hat{\tp}_j}{\tp_j}}=O\Paren{\sqrt{\frac{(\ab/\nspu)^2}{(\ns/\nspu)e^\priv}}}=O\Paren{\sqrt{\frac{\ab^2}{\nspu\ns e^\priv}}}.
$$
Plugging in ~\eqref{equ:decompose} yields the desired bound. Detailed proofs of~\cref{lem:decompose} and~\cref{thm:small_m} are in~\cref{app:small_m}. }

\subsection{Medium $\nspu$ ($\ab/e^\priv<\nspu< \ab$) }
In this regime, we discover that increasing $\nspu$ barely helps with improving the error rates in certain parameter regimes.
\begin{theorem}
\label{thm:mid_m}
For $\priv > 1$ and $\ab/e^\eps<\nspu < \ab$, if $\ns > C \nspu \log(\nspu)/\priv$ for some constant $C$, we have
\[
\risk(\eps, \ab, \ns, \nspu) = O\Paren{ \sqrt{\frac{\ab}{\nspu\ns}} \lor \sqrt{\frac{\ab\ln(\ab/\nspu+1)}{\ns \priv}}}.
\]
The bound is achieved by a non-interactive protocol.
\end{theorem}

Note that $\risk(\eps, \ab, \ns, \nspu)$ is non-increasing with $\nspu$. Setting $\nspu=\ab$ in~\cref{thm:lb-large-m} yields a lower bound of $\Omega\Paren{\sqrt{{\ab/\nspu\ns}}\lor \sqrt{k/\ns\priv}}$ for $\ab/e^\eps<\nspu<\ab$ when $\ns > (\ab/\eps)^2$. Thus \cref{thm:mid_m} is tight up to logarithmic factors.

When $\nspu \le \ab/e^{\priv/2}$, we use the same algorithm as $\nspu\le \ab/e^\priv$, and the guarantee is proved similarly (see~\cref{app:medium_m} for details). When $\nspu >\ab/e^{\priv/2}$, we make the following changes,
\begin{enumerate}
    \item To learn $\p_B=[\p(B_1), \ldots, \p(B_\nspu)]$, we use $\eps/2$ privacy budget with the algorithm for $\nspu\ge \ab$. Hence, the estimation error for $\p_B$ satisfies
    \ifnum\doublecolumn=1
        \begin{align*}
            \expect{\totalvardist{\hat{\p}_B}{\p_B}}&= O\Paren{ \sqrt{\frac{\nspu^2}{\nspu\ns\eps}}}=O \Paren{ \sqrt{\frac{\ab}{\ns\eps}}}.
        \end{align*}  
    \else
        \[
        \expect{\totalvardist{\hat{\p}_B}{\p_B}}= O\Paren{ \sqrt{\frac{\nspu^2}{\nspu\ns\eps}}} = O\Paren{ \sqrt{\frac{\nspu}{\ns\eps}}}=O \Paren{ \sqrt{\frac{\ab}{\ns\eps}}}.
        \]
    \fi
The final equality is due to $\nspu<\ab$.

\item To estimate $\p(\cdot|B_j)$, we divide the remaining budget of $\priv/2$ into $t'\eqdef \lfloor\frac{\priv}{2\ln (\ab/\nspu)}\rfloor$ parts. Note that with $\ln(\ab/\nspu)$ privacy budget and $\ns/\nspu$ samples, we can learn $\tp_j$ with accuracy $O(\sqrt{\ab/\ns})$.
Since $\ab/\nspu<e^{\eps/2}$, we can assign $\nspu\land t'$ blocks to each user. The effective sample size increases by a factor of $\nspu\land t'$. Thus 
    \ifnum\doublecolumn=1
        \begin{align*}
             \expect{\totalvardist{\hat{\tp}_j}{\tp_j}}&=O\left(\sqrt{\frac{\ab}{\ns (\nspu\land t')}}\right)\\
             &= O\left(\sqrt{\frac{\ab}{\nspu\ns}}\lor\sqrt{\frac{\ab\ln(\ab/\nspu+1)}{\ns \priv}}\right).
        \end{align*}
    \else
        \[
         \expect{\totalvardist{\hat{\tp}_j}{\tp_j}}=O\left(\sqrt{\frac{\ab}{\ns (\nspu\land t')}}\right)= O\left(\sqrt{\frac{\ab}{\nspu\ns}} \lor\sqrt{\frac{\ab\ln(\ab/\nspu+1)}{\ns \priv}}\right).
        \]
    \fi
\end{enumerate}
Applying~\cref{lem:decompose} yields the desired upper bound.

\section{Lower bound}
\label{sec:lower-bound}
\begin{figure*}[ht]
    \centering
    \includegraphics[width=0.3\linewidth]{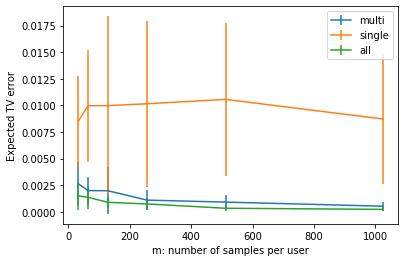}
    \includegraphics[width=0.3\linewidth]{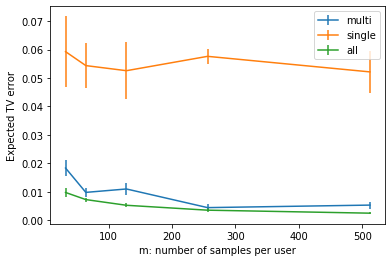}
    \includegraphics[width=0.3\linewidth]{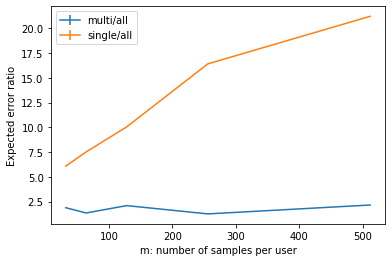}    
    \caption{\textbf{Left and Middle}: expected TV error with respect to $\nspu$ of our algorithm (blue), 1-sample HR (orange), and all-sample HR (green) in the high privacy regime with $\priv=0.9$. \textbf{Left}: $ \ab=2, \ns = 9000, \p=(0.6, 0.4)$. \textbf{Middle}: $\ab=32, \ns = 9000\ab, \p$ uniform. \textbf{Right}: orange/green and blue/green ratio in the middle plot.} 
    \label{fig:eps_small}
\end{figure*}
\begin{figure*}[ht]
    \centering
    \includegraphics[width=0.3\linewidth]{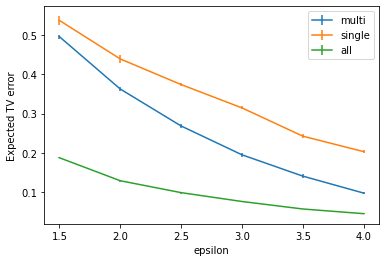}
    \includegraphics[width=0.3\linewidth]{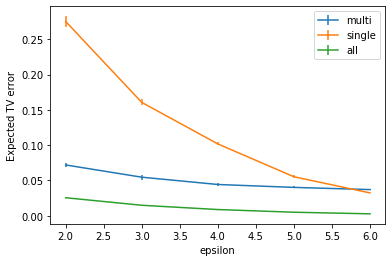}    \includegraphics[width=0.3\linewidth]{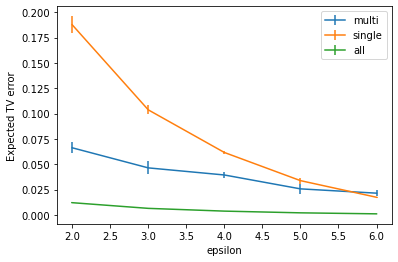}
    \caption{Expected TV error with respect to $\priv$ of our algorithm (blue), 1-sample HR (orange), and all-sample HR (green) in the $\eps>1$ regime with $\p$ uniform. \textbf{Left: $\nspu< \ab/e^{\priv}$}, $\ab=1000, \nspu=20, \ns = 600\ab$. \textbf{Middle: $\nspu\in[\ab/e^{\priv}, \ab)$. } $\ab=500, \nspu=128, \ns = 1200\ab$. \textbf{Right: $\nspu\ge \ab$}, $\ab=200, \nspu=256, \ns = 1200\ab$.}
    \label{fig:low_priv_wrt_eps}
\end{figure*}
In this section, we discuss the proof of lower bounds in \cref{thm:rate_small_eps}, \cref{thm:small_m}, and \cref{thm:lb-large-m}. 
We use the information contraction framework in \cite{acharya2020unified} and the lower bound construction in \cite{ACLST:21}. 
Our hard instances are from the ``Paninski'' family~\cite{Paninski:08}.
Let $\gamma\in(0, 1/2)$ be a parameter related to the expected error. We consider a family of distributions defined as follows: for each vector $z\in\mathcal{Z}\eqdef\{-1, 1\}^{k/2}$, define a discrete distribution $p_z$ as
\[
p_z(2i-1)=\frac{1+\gamma z_i}{k}, \quad p_z(2i)=\frac{1-\gamma z_i}{k}.
\]

The $\nspu$ samples observed by each user can be viewed as a $\ab$-dimensional vector indicating the histogram from a multinomial distribution. 
While the proof builds on \cite{ACLST:21, acharya2020unified} for the communication-constrained case, their techniques cannot directly translate to $\eps$-LDP. Their proof relies on the ``Poissonization'' trick to make each coordinate independent. However, for the trick to work,  each user needs to send one bit to indicate whether they get enough samples, which might violate privacy constraints. Our solution is as follows,
\begin{itemize}
    \item For $\nspu \le \ab/e^\eps$ and $\eps<1$, we compute the information contraction bound for multinomial distributions directly, which leads to tight lower bounds \newer{without ``Poissonization''}.
    \item For $\nspu > \ab/e^\eps$ and $\eps > 1$, ``Poissonization'' is still used. 
    We show that 
    even if we allow each user to send an extra clean bit, which we term ``$\eps$-LDP $+$ 1-bit'' channels, the desired lower bound still holds. 
\end{itemize}
We defer the detailed proof to~\cref{app:lower}.

\section{Experiments}
\label{sec:experiments}
\newer{The main goal of the section is to demonstrate the effectiveness of our algorithmic ideas and verify our theoretical findings. The experiments are based on prototype algorithms without extensive tuning on constants.} We mainly focus on the \textit{interactive} versions since they give better constants than the non-interactive ones numerically.
We compare our algorithms to Hadamard Response~\cite{ASZ:18} with 1 sample per user on either $\ns$ users (referred as \textbf{1-sample HR}) or $\nspu\ns$ users (\textbf{all-sample HR}) in various parameter regimes. They serve as baseline upper and lower bounds on the achievable rates under user-level LDP \footnote{\newer{RAPPOR \cite{EPK:14} outperforms HR numerically by a small margin (e.g., \cite{ASZ:18}). We compare with HR here since our algorithms use ideas from HR and both algorithms are optimal up to constants.}}. Average TV error and standard deviation over 5 independent runs are reported. Additional results for both interactive and non-interactive algorithms are provided in~\cref{app:experiment}.

\textbf{High privacy $ \priv\le 1$}.
Figure~\ref{fig:eps_small} shows the result for the high privacy regime for $\ab=2$ and $\ab=32$, with $\nspu=[32, 64, 128, 256, 512]$. In both cases, the performance of the 1-sample HR remains nearly the same, while the performance of our algorithm is always within a \newer{constant (2.5)} factor to that of all-sample HR, as Theorem~\ref{thm:rate_small_eps} suggests.

\textbf{Low privacy $\priv>1$}. In this regime, we mainly focus on the dependence on $\eps$. Figure~\ref{fig:low_priv_wrt_eps} shows the expected TV error with respect to $\priv$. When $\nspu<\ab/e^\priv$, our algorithm approaches all-sample HR as $\priv$ increases. The rate of decay is much faster than 1-sample HR, indicating an exponential decay with $\priv$ as suggested by Theorem~\ref{thm:small_m}. 

When $\nspu\ge \ab/e^{\priv}$, as suggested by~\cref{thm:mid_m} and~\cref{cor:large_m}, our algorithms no longer improve exponentially with $\priv$ and gradually approaches 1-sample HR (near $\eps=6$). This is expected, as their rates differ by at most a factor of $\Theta(\sqrt{\ln k})$ when $\nspu=\Theta(\ab)$ and $\eps\simeq \ln k$. 



\bibliographystyle{plain} 
\bibliography{bibliography}

\newpage
\appendix

\section{Detailed algorithm and proof for $\priv < 1$}
\subsection{$\priv < 1, k=2$}
\label{app:binomial}
In this section, we provide the detailed proof of \cref{thm:binomial_ldp}. Recall the coin estimation problem: there are $\ns$ users, each has $\nspu$ i.i.d. samples from $\bernoulli{p}$. We want to estimate $p$ under LDP. We start with the necessary definitions.

Let $\CI=10$ and $r\eqdef \lfloor\sqrt{\frac{\nspu}{2\CI}}\rfloor$. We now recall the 
definition of the intervals $\{I_i\}_{i \in [2\nin]}$. Let
$I_i\eqdef[l_{i-1}, l_{i}]$ for $1\leq i\leq r$, where
\[
  l_i\eqdef \min\left\{\frac{\CI i^2}{\nspu}, \frac{1}{2}\right\}, \quad 0\le i\le r.
\]
Furthermore $I_{2r+1-i}\eqdef[1-l_i, 1-l_{i-1}]$. 

Let $\CR=100\CI^2$ and $r'=\lfloor\sqrt{\frac{\nspu}{2\CR}}\rfloor$. For $1\le i\le r'$, define $ L_i=[l_{i-1}'-l_{i}'] $ similarly as $\{I_i\}_{i\in [2r]}$ with $\CI$ replaced by $\CR$. Let $j_i=(l_{i-1}'+l_i)/2$ and $\{J_i\}_{i 
\in 
[2\nin + 1]}$ be the partition defined by $j_i$s.

For user $u$, let $Z_u$ be the number of 1s. $Z_u$ induces a random variable $V_u:=\arg\max_{i\in [2r]}\indic{Z_u\in I_i}$, which follows a discrete distribution $\q$, with 
$
\q_i=\Pr[Z_u\in I_i], \;i\in[2r].
$

Recall that $\beta=1/(e^\priv+1)$ and define $\gamma=1-2\beta=\frac{e^\priv-1}{e^\priv+1}$. The detailed protocol is described in~\cref{alg:non-interactive}. 

In~\cref{alg:non-interactive} we define the functions $R_2, R_3, R_4$ as
\begin{equation}
  R_2(p)\eqdef \bPr{\frac{Z}{\nspu}\in\bigcup_{i}L_{2i}}, \quad
  R_3(p)\eqdef \bPr{\frac{Z}{\nspu}\in\bigcup_{i}J_{2i}}, \quad
  R_4(p)=\bPr{Z\geq 1},
\end{equation}
where $Z\sim\binomial{\nspu}{p}$.
\begin{algorithm}[h]
\caption{Non-interactive binomial Estimation Protocol.}

Divide users into 4 groups $S_1, 
\ldots, S_4$. $|S_1|=\frac{\ns}{2}=:\nloc$, $|S_2|=|S_3|=|S_4|=\frac{\ns}{6}=:N$.

\medskip
\textbf{Localization stage.}  In this stage, the goal is to obtain an interval 
$I$, which corresponds to a crude 
estimate of $p$.
\begin{itemize}
	\item \textbf{Users}: $u\in S_1$ computes the one-hot encoding of $V_u$ and flips each coordinate with probability $\beta := 1/(e^\priv + 1)$. Denote the flipped vector as $Y_u:=(Y_{u, 1}, \ldots, Y_{u, 2r})$. 
	\item \textbf{The server}: Let 
    \[
        \hat{i}=\arg\max_{i\in[2r]}\sum_{u\in S_1}Y_{u, i}.
    \]
    Set the confidence interval $\hat{I}=\cup_{i:|i-\hat{i}|\le 1}I_i$.
\end{itemize}

\textbf{Refinement stage.} In this stage, we improve the accuracy to 
$\Theta(1/\nspu\ns)$. 
\begin{itemize}
	\item \textbf{Users}: 
	\begin{enumerate}
		\item $u\in S_2$ flip $\indic{Z_u/\nspu\in 
			\cup_{i}L_{2i}}$ with probability $\beta$.
		\item 	$u\in S_3$ flip $\indic{Z_u/\nspu\in \cup_{i} 
		J_{2i}}$ with probability $\beta$. 
		\item $u\in S_4$ flips $\indic{Z_u\ge 1}$ with probability $\beta$. 
	\end{enumerate}
	Denote $Y_u$ as the flipped bit and send $Y_u$
	\item \textbf{The server}: According to \cite[Lemma A.8]{ACLST:21}, one of the 3  cases must hold.
	\begin{description}
		\item[If]  $\hat{I}\subseteq[0, 65\CR/m]$, let $\bar{Y}_4 = 
		\Paren{\frac{1}{N}\sum_{u\in 
		S_4}Y_u-\beta}/\gamma$
		\[
		\hat{p}=R_4^{-1}\left(\bar{Y}_4 \right) \eqdef \setOfSuchThat{p \in [0,1] }{ R_4(p) }
		=\bar{Y}_4\}.  
		\]
		\item[Else if] there exists $i\in [2r]$ such that
		$
		\hat{I}\subseteq I'_i\eqdef \left[l_i'-\frac{0.55\CR i}{\nspu}, l_i'+\frac{0.55\CR 
			i}{\nspu}\right],
		$ let $\bar{Y}_2 =\Paren{\frac{1}{N}\sum_{u\in 
			S_2}Y_u-\beta}/\gamma$ 
		\[
		\hat{p}=R_{2, I_i'}^{-1}\left(\bar{Y}_2\right) \eqdef \setOfSuchThat{p \in I_i' }{ R_2(p) 
		=\bar{Y}_2 }.
		\]
		\item[Else if] there exists $i\in [2r+1]$ such that
		$
		\hat{I}\subseteq J'_i\eqdef \left[j_i-\frac{0.55\CR i}{\nspu}, j_i+\frac{0.55\CR 
			i}{\nspu}\right],
		$ let $\bar{Y}_3 = \Paren{\frac{1}{N}\sum_{u\in 
			S_3}Y_u-\beta}/\gamma$
		\[
		\hat{p}=R_{3, J_i'}^{-1}\eqdef \setOfSuchThat{p \in J_i' }{ R_3(p) =\bar{Y}_3 }.
		\]
	\end{description}
\end{itemize}
\label{alg:non-interactive}

\end{algorithm}

\paragraph{Guarantee}  First we prove the guarantee of the localization stage. We start with the following observation about the partition $\{I_i\}_{i\in[2r]}$.
\begin{lemma}
\label{lem:partition-property}
Suppose that $p\in I_i$ and $p\le 1/2$. Then $$\max\{\frac{\CI}{\nspu}, \frac{5}{3}\sqrt{\CI}\sqrt{\frac{p}{\nspu}}\}\le |I_i|\le \max\{\frac{\CI}{\nspu}, 2.5\sqrt{\CI}\sqrt{\frac{p}{\nspu}}\}$$
\end{lemma}
\begin{proof}
If $i=1$, then $|I_i|=\CI/\nspu$. If $i\ge 2$, then since $p\ge \CI i^2/\nspu$,
$$
|I_i|=\frac{\CI(2i+1)}{\nspu}\le \frac{\CI(2.5i)}{\nspu}\le2.5\sqrt{\CI}\sqrt{\frac{p}{\nspu}}.
$$
Since $p\le \CI (i+1)^2/\nspu$ and $1\le (i+1)/3$,
\[
|I_i|=\frac{\CI(2i+1)}{\nspu}\ge \frac{\CI(2-1/3)(i+1)}{\nspu}=\frac{5}{3}\sqrt{\CI}\sqrt{\frac{p}{\nspu}}.
\]
\end{proof}

\begin{theorem}
\label{thm:priv_local}
Recall that $\hat{I}=\cup_{|i-\hat{i}|\le 1}I_i$. There exists a constant $C$ such that if $\nloc\ge C\log (\nspu)/\priv^2$, we have
\[
\expect{(\hat{p}-p)^2\indic{p\notin \hat{I}}}= O\left(\frac{1}{m\nloc\priv^2}\right)
\]
\end{theorem}
\begin{proof}
Let $i_p$ be such that $p\in I_{i_p}$. Due to \cref{lem:partition-property}, by concentration inequalities for binomials
\[
\sum_{|i-i_p|\le 1}\q_{i}\ge 0.96=:1-\alpha.
\]

For $i\in[2r]$, let $M_i=\sum_{u\in S_1}Y_{u, i}$. By Chernoff bound, for $i$ such that $|i-i_p|>1$, with probability at least $1-\delta$, 
\[
M_i\le \nloc(\beta +\alpha\gamma)+\sqrt{3\nloc(\beta+\alpha\gamma)\log\left(\frac{1}{\delta}\right)}=:M^*. 
\]

By union bound, with probability at least $1-2r\delta = 1-\Theta(\sqrt{\nspu}\delta)$, $M_i\le M^*$ for all $|i-i_p|>1$. 

Let $i^*=\arg\max_{i}\q_i$. It is clear that $|i^*-i_p|\le 1$, and $\q_{i^*}\ge (1-\alpha)/3=0.32$.
Next we argue that with high probability, $M_{i^*}>M^*$, 
and hence the maximum of $M_i$'s must be achieved at some $i\in[i_p-1, i_p+1]$. 

First, there exists a constant $C_1$ such that when $\nloc\ge C_1\log(1/\delta)/\priv^2$, 
\[
\expect{M_{i^*}-M^*}=\frac{1-4\alpha}{3}\nloc\gamma-\sqrt{3\nloc(\beta+\alpha\gamma)\log\frac{1}{\delta}}\ge \frac{1-4\alpha}{6}\nloc\gamma.
\]
Therefore, by Chernoff bound, 
\begin{align*}
    \probaOf{M_{i^*}\le M^*}&=\probaOf{M_{i^*}\le \expect{M_{i^*}}-\expect{M_{i^*}-M^*}}\\
    &\le \exp\left(-\frac{\gamma^2 0.28^2}{\beta+0.32\gamma}\nloc\right), 
\end{align*}
which is at most $\delta$ as long as $\nloc\ge C_2\log(1/\delta)/\priv^2$ for some constant $C_2$. 

Set $\delta=\frac{1}{\nspu^2\nloc\priv^2}$ and $C'=\max\{C_1, C_2\}$. Then 
\[
\expect{(\hat{p}-p)^2\indic{p\notin\hat{I}}}\le \probaOf{p\notin\hat{I}}\le (\sqrt{\nspu}+1)\delta=O\left(\frac{1}{\nspu\nloc\priv^2}\right), 
\]
as long as 
$$\nloc\priv^2\ge C'\log(\nspu^2\nloc \priv^2)=2C'\log \nspu +C'\log(\nloc \priv^2).$$ 
In addition, there exists a constant $C_3$ such that as long as $\nloc\priv^2\ge C_3$, we can guarantee $\nloc\priv^2/2\ge C'\log(\nloc \priv^2)$. Hence, let $C=\max\{C', C_3\}$, we can guarantee the desired error as long as $\nloc\ge C\log(\nspu)/\priv^2$
\end{proof}

Next we show the guarantee of the refinement stage. Using a similar argument in the proof of \cite[Lemma A.8]{ACLST:21}, the interval $\hat{I}$ has the following property.

\begin{lemma}
	\label{lem:large-deriv-intervals}
	Conditioned on $p\in\hat{I}$, at least one of the following must hold,
	\begin{enumerate}
		\item There exists $i\in [2r]$, such that $
		\hat{I}\subseteq I_i'=\left[\frac{\CR i^2}{\nspu}-\frac{0.55\CR i}{\nspu}, 
		\frac{\CR 
			i^2}{\nspu}+\frac{0.55\CR i}{\nspu}\right]
		$
		\item There exists $i\in [2r+1]$ such that $\hat{I}\subseteq 
		J_i'=\left[j_i-\frac{0.55\CR i}{\nspu}, j_i+\frac{0.55\CR 
			i}{\nspu}\right]$
		\item $J\subseteq [0, 65\CR/\nspu]$   
	\end{enumerate}
\end{lemma}

The proof is identical to \cite[Lemma A.8]{ACLST:21}. 
Furthermore, in the respective intervals stated in~\cref{lem:large-deriv-intervals}, there is at least one of $R_2(p), R_3(p), R_4(p)$ with large derivatives.
\begin{lemma}
\label{lem:refine-derivative}
There exists some absolute constant $C>0$ such that the following holds.
\begin{enumerate}
	\item For all $i \in [2\nin]$, $R_2(p)$ is monotonic in 
	$I'_i\eqdef\left[l_i'-\frac{0.55\CR 
	i}{\nspu}, 
	l_i'+\frac{0.55\CR i}{\nspu}\right]$, and for $p \in I'_i$, 
	\[
		\left|\frac{dR_2(p)}{dp}\right|\ge C\sqrt{\frac{\nspu}{p}}.
	\]
	\item  For all $i \in [2\nin + 1]$, $R_3(p)$ is monotonic in 
	$J'_i\eqdef\left[j_i-\frac{2\CR i}{\nspu}, j_i+\frac{0.55\CR 
	i}{\nspu}\right]$, and for  $p \in J'_i$,
		\[
	\left|\frac{dR_3(p)}{dp}\right|\ge C\sqrt{\frac{\nspu}{p}}.
	\]
	\item $R_4(p)$ is monotonic in $[0, 65\CR/m]$, and for $p \in [0, 
	65\CR/m]$,
	\[
		\frac{dR_4(p)}{dp}\ge C\nspu.
	\]
\end{enumerate}
\end{lemma}
The proof is identical to \cite[Lemma A.9]{ACLST:21}

\begin{proof}[Proof of Theorem~\ref{thm:binomial_ldp}]
 We note that for $h\in\{2, 3, 4\}$, using the analysis for Randomized response, 
 \[
 \expect{(\bar{Y}_h-R_h(p))^2}\le \frac{1}{N}\cdot\Paren{\frac{e^\priv+1}{e^\priv -1}+\frac{e^\priv+1}{(e^\priv-1)^2}}=O\Paren{\frac{1}{N\priv^2}}
\]
 
Using the same analysis in \cite[Theorem A.3]{ACLST:21}, we have
\[
\expect{(\hat{p}-p)^2\indic{p\in \hat{I}}}=O\left(\frac{1}{\nspu\ns\priv^2}\right).
\]
Combining with Theorem~\ref{thm:priv_local} yields the desired result.
\end{proof}

\subsection{$\priv<1, k>2$}
\label{sec:small_eps_general_k}
\begin{theorem}
\label{thm:privacy_main}
There exists a constant $C$ and an $\priv$-LDP algorithm such that when $\ns\ge C \ab\log(\nspu)/\priv^2$, 
\[
\expect{\totalvardist{\hp}{\p}}=O\left(\sqrt{\frac{\ab^2}{\nspu\ns\priv^2}}\right)
\]
\end{theorem}

\begin{proof}
We use an idea considered in~\cite{ASZ:18} and  
estimate
the 
probabilities of 
subsets of $[\ab]$ 
defined 
below. 

Let  $K = 2^{\clg{\log_2 (k+1)}}$ be the smallest power of 2 larger 
than $\ab$ and $H_K$ be the $K \times K$ Hadamard matrix. Define $T_i = 
\{j \in [\ab]: H_K(i, j) = 1\}$, i.e., the locations of 1's in the $i$th row of 
$H_K$. Let $\p_T = (\p(T_1), \p(T_2), \ldots, \p(T_K))$. The following two 
claims are shown 
in~\cite{AS:19}.
\begin{claim} \label{clm:ptopt}
	 For any distribution 
	$\p$, we have
	\[
	\p_T = \frac{H_K \cdot \p + \mathbf{1}_K}{2},
	\]
	where we append 0's to $\p$ to make it of dimension $K$.
\end{claim}
\begin{claim}\label{clm:errptoerrpt}
	For all $\p, \hat{\p}$, we have 
	\[
		\norm{\p_T -\hat{\p} _T}_2^2 = \frac{K}{4} \norm{\p - \hat{\p} }_2^2.
	\]
\end{claim}

The above two claims show that for any estimate for the set probabilities 
$\hat{\p}_T$, we can obtain an estimate for $\p$ by inverting the formula 
in Claim~\ref{clm:ptopt}. Moreover, Claim~\ref{clm:errptoerrpt} establishes 
the relation between the errors for the two estimates.

Now we described the protocol.
\begin{itemize}
	\item[1.] Divide users into $K$ subsets, each with size $\ns/K$.
	\item[2.] Users in the $i$th set count the number of samples in $T_i$ 
	and apply the $\priv$-LDP protocol in Theorem~\ref{thm:binomial_ldp} to 
	estimate $\p(T_i)$.
	\item[3.] After obtaining the estimates $\hat{\p}(T_i)$, the server returns 
	the first $\ab$ coordinates of
	$\hat{\p}$ where
\[
	\hat{\p} = H_K^{-1}\Paren{ 2 \hat{\p}_T - \mathbf{1}_K},
\]
where $\hp_T = (\hp(T_1), \hp(T_2), \ldots, \hp(T_K))$.
\end{itemize}

By Theorem~\ref{thm:binomial_ldp}, for $\ns/K\ge C\log \nspu/\priv^2$ where $C$ is the constant in \cref{thm:binomial_ldp}, 
\[
	\bEE{\norm{\hp_T - \p_T}_2^2} = \sum_{i = 1}^\ab \bEE{\Paren{\hp(T_i) 
	- \p(T_i)}^2}  = O\Paren{\frac{K}{\nspu (\ns/K) \priv^2}} = 
	O\Paren{\frac{K^2}{\nspu \ns \priv^2}}.
\]

Combining with Claim~\ref{clm:errptoerrpt}, we get
\[
	\bEE{\totalvardist{\hp}{\p}}  \le \frac{1}{2}\sqrt{K\bEE{\norm{\hat{\p} - 
				\p}_2^2}} = \frac{1}{2}\sqrt{K\bEE{\frac{4}{K}\norm{\hat{\p}_T - 
				\p_T}_2^2}} = O\Paren{\sqrt{\frac{K^2}{\nspu \ns \priv^2}}}.
\]
Then the upper bound of Theorem~\ref{thm:privacy_main} follows by $K \le 2\ab$.

\end{proof}

\section{Missing proofs for $\priv>1$}
\label{app:low_priv}
\subsection{$m\le k/e^{\priv}$}
To prove \cref{lem:decompose}, we use ~\cite[Lemma 3.2]{ACLST:21}, which states
\[
\expect{\totalvardist{\hp}{\p}}\le\expect{\totalvardist{\hat{\p}_B}{\p_B}}+\sum_j \p(B_j)\totalvardist{\hat{\p}_j}{\bar{\p}_j}
\]

The only missing part is the following claim.
\begin{claim} 
For all $j \in [t]$,
\[
    \totalvardist{\hat{\p}_j}{\bar{\p}_j} \le \frac{\totalvardist{\hat{\tp}_j}{\tp_j}}{1 - \tp_j(\bot)}.
    \]
\end{claim}
\begin{proof}
\begin{align}
    \totalvardist{\hat{\p}_j}{\bar{\p}_j} & = \sum_{x \in B_j} \abs{\hat{\p}_j(x) - \bp_j(x)} \\
    & = \sum_{x \in B_j} \abs{ \frac{\hat{\tp}_j(x)}{1 - \hat{\tp}_j(\bot)} -  \frac{\tp_j(x)}{1 - \tp_j(\bot)}} \\
    & \le \sum_{x \in B_j}  \Paren{\abs{ \frac{\hat{\tp}_j(x)}{1 - \hat{\tp}_j(\bot)} -  \frac{\hat{\tp}_j(x)}{1 - \tp_j(\bot)}} + \abs{ \frac{\hat{\tp}_j(x)}{1 - \tp_j(\bot)} -  \frac{\tp_j(x)}{1 - \tp_j(\bot)}}} \\
    & = \sum_{x \in B_j}  \frac{\hat{\tp}_j(x) \abs{\hat{\tp}_j(\bot) - \tp_j(\bot)} }{\Paren{1 - \hat{\tp}_j(\bot)}\Paren{1 - \tp_j(\bot)}}
    + \frac{\sum_{x \in B_j} \abs{ \hat{\tp}_j(x) - \tp_j(x) }}{1 - \tp_j(\bot)} \\
    & = \frac{\abs{\hat{\tp}_j(\bot) - \tp_j(\bot)} + \sum_{x \in B_j}\abs{\hat{\tp}_j(x) - \tp_j(x)}  }{1 - \tp_j(\bot)} \\
    & = \frac{\totalvardist{\hat{\tp}_j}{\tp_j}}{1 - \tp_j(\bot)}.
\end{align}

\end{proof}
Noting that $1-\tp_j(\bot)=\Theta( m\p(B_j)\land 1)$ completes the proof of~\cref{lem:decompose}.

Finally to prove \cref{thm:small_m}, recall that 
\begin{equation*}
     \quad \totalvardist{\hat{\tp}_j}{\tp_j}=O\Paren{\sqrt{\frac{(\ab/\nspu)^2}{(\ns/\nspu)e^\priv}}}=O\Paren{\sqrt{\frac{\ab^2}{\nspu\ns e^\priv}}}.
\end{equation*}

Therefore,
\begin{align*}
    &\quad\sum_{j \in [t]} \p(B_j) \expect{\totalvardist{\hat{\p}_j}{\bar{\p}_j} }  \le O\Paren{\sum_{j \in [t]} \frac{\p(B_j)}{ m\p(B_j)\land 1}\expect{ \totalvardist{\hat{\tp}_j}{\tp_j} }}  \label{equ:hadamard_step}\\
    & = O\Paren{\sqrt{\frac{\ab^2}{\nspu \ns e^{\priv}}}} \cdot \sum_{j \in [\nspu]} \Paren{\p(B_j)+\frac{1}{m}} =  O\Paren{\sqrt{\frac{\ab^2}{\nspu \ns e^{\priv}}}}.
\end{align*}
Combining with~\eqref{equ:block_decomposition}
completes the proof of \cref{thm:small_m}.

\ignore{
\begin{enumerate}
    \item Let $\priv_0=0.5$. Divide the domain $[\ab]$ into $t = \nspu$ blocks $B_1, \ldots, B_t$, each with size $
    \ab/\nspu$.
    \item Each user uses $\priv_0=0.5$ to estimate the block distribution $\p_B\eqdef [\p(B_1), \ldots, \p(B_t)]$ with the algorithm for $\priv\le 1$ in Section~\ref{sec:small_eps_general_k}.
    \item Divide all users into $t$ groups. The $j$th group tries to estimate $\bar{\p}_j:=\p(\cdot |B_j)$,
    the distribution conditioned on a sample is in $B_j$ (treated as uniform if $\p(B_j)=0$). \ie, for $x \in B_j$
    \[
        \bar{\p}_j(x) = \frac{\p(x)}{\p(B_j)}.
    \]
    To do this, each user in the $j$th group considers the following distribution $\tilde{{\p}}_j$ over $B_j \cup \{\bot\}$ where
    \[
        \tp_j(\bot) := \bP{X^\nspu \sim \p}{ \forall x \in B_j, x \notin X^\nspu} = \Paren{1 - \p(B_j)}^\nspu.
    \]
    And for $x \in B_j$, $\tp_j(\bot)(x)$ is the probability that $x$ is the first symbol in $B_j$ that appears in $X^\nspu$ conditioned on that $\exists x' \in B_j, x' \in B_j$. It can be otained that
    \[
        \tp_j(x) = \bp_j(x) \Paren{1 - \tp(\bot)}.
    \]
    It can be verified that a user can simulate a sample from $\tp_j$ by getting $\bot$ if $\forall x \in B_j, x \notin X^\nspu$ and getting the first sample in $X^m \cap B_j$ if it is not empty. Now each user can use the Hadamard response scheme in \cite{ASZ:18} for $(\eps - \eps_0)$-LDP to send messages to the server.
    
    The server can then get an estimate $\hat{\tp}_j$ for $\tp_j$ using the messages above. Using $\hat{\tp}_j$, an estimate $\hp_j$ for $\bp_j$ can be obtained by  $\forall x \in B_j$
    \[
        \hp_j(x) = \frac{\hat{\tp}_j(x)}{1 - \hat{\tp}_j(\bot)}.
    \]
\end{enumerate}

\begin{theorem}
   For $\priv>1$ and $m\le k/e^\priv$, there exists an algorithm with an output $\hat{\p}$ that achieves 
   \[
   \expect{\totalvardist{\hat{\p}}{\p}}=O\left(\sqrt{\frac{\ab^2}{\nspu\ns e^\priv}}\right).
   \]
\end{theorem}
\begin{proof}
The idea is the following:
The total variation error can be decomposed into 
$$
\expect{\totalvardist{\hat{\p}_B}{\p_B}}+\sum_j \p(B_j)\totalvardist{\hat{\p}_j}{\bar{\p}_j}
$$
By the guarantee of the $\eps_0$-LDP algorithm, We have
$$
\expect{\totalvardist{\hat{p}_B}{p_B}}= O\Paren{ \sqrt{\frac{m^2}{mn}}} = O\Paren{ \sqrt{\frac{m}{n}}}.
$$
Whenever $\nspu < \ab/e^{\eps/2}$, this term is always smaller than the desired bound, and hence naturally holds for $\nspu<\ab/e^\eps$. Next we bound the second term.

By the guarantee of the Hadamard response algorithm, we get 
\begin{equation}\label{eqn:bound_hadamard}
    \expect{\totalvardist{\hat{\tp}_j}{\tp_j}} = O\Paren{\sqrt{\frac{(\ab/\nspu)^2}{(\ns/t) e^{\priv - \priv_0}}}} = O\Paren{\sqrt{\frac{\ab^2}{\nspu \ns e^{\priv}}}}.
\end{equation}
The following claim relates $\totalvardist{\hat{\p}_j}{\bar{\p}_j}$ to $\totalvardist{\hat{\tp}_j}{\bar{\tp}_j}$.

Combining \cref{eqn:bound_hadamard} and \cref{clm:tp2bp}, we get
\begin{align}
    \sum_{j \in [t]} \p(B_j) \expect{\totalvardist{\hat{\p}_j}{\bar{\p}_j} } & \le \sum_{j \in [t]} \frac{\p(B_j)}{1 - \tp_j(\bot)}\expect{ \totalvardist{\hat{\p}_j}{\bar{\p}_j} } \\
    & = O\Paren{\sqrt{\frac{\ab^2}{\nspu \ns e^{\priv}}}} \cdot \sum_{j \in [t]} \frac{\p(B_j)}{1 - \tp_j(\bot)}\\
    & \le  O\Paren{\sqrt{\frac{\ab^2}{\nspu \ns e^{\priv}}}} \cdot \sum_{j \in [2\nspu]} \frac{\p(B_j)}{\min \{ m\p(B_j), 1\}} \\
    & =  O\Paren{\sqrt{\frac{\ab^2}{\nspu \ns e^{\priv}}}}.
\end{align}
This completes the proof.
\end{proof}
}

\subsection{$\ab/e^\eps<\nspu<\ab$}
\label{app:medium_m}
We provide the detailed proof for $m\le \ab/e^{\eps/2}$. Recall that in this regime we use the algorithm for $m\le k/e^\eps$. Since $m\le \ab/e^{\eps/2}$, by \cref{thm:rate_small_eps}
\[\expect{\totalvardist{\hat{\p}_B}{\p_B}}=O\Paren{\sqrt{\frac{\nspu}{\ns}}}=O\Paren{\sqrt{\frac{\ab}{\ns e^{\priv/2}}}}
\]
However, since each block only has $\ab/\nspu\le e^{\eps}$ elements, the error for estimating $\tilde{\p}_j$ satisfies
\[
 \expect{\totalvardist{\hat{\tp}_j}{\tp_j}} = O\Paren{\sqrt{\frac{\ab/\nspu}{(\ns/t) }}} = O\Paren{\sqrt{\frac{\ab}{\ns }}}.
\]
Using the same argument as $\nspu\le \ab/e^{\priv}$, we have
\[
\expect{\totalvardist{\hat{\p}}{\p}}=O\left(\sqrt{\frac{\ab}{\ns }}\right)= O\left(\sqrt{\frac{\ab\ln(\ab/\nspu+1)}{\ns \priv}}\right).
\]
The final equality is due to $\priv/2\le \ln (k/m)$.
\ignore{
Next we consider $\nspu>\ab/e^{\eps/2}$. In this regime, we make the following changes to the algorithm for $\nspu\le \ab/e^\eps$,
\begin{enumerate}
    \item When learning the block distribution $\p_B=[\p(B_1), \ldots, \p(B_t)]$, we use $\eps/2$ privacy budget with the algorithm for $\nspu\ge \ab$. Hence, the estimation error for $\p_B$ satisfies
\[
\expect{\totalvardist{\hat{\p}_B}{\p_B}}= O\Paren{ \sqrt{\frac{\nspu^2}{\nspu\ns\eps}}} = O\Paren{ \sqrt{\frac{\nspu}{\ns\eps}}}=O \Paren{ \sqrt{\frac{\ab}{\nspu\eps}}}.
\]
The final inequality is due to $\nspu<\ab$.

\item When estimating the conditional distributions $\p(\cdot|B_j)$, we split divide remaining privacy budget into $\ln(\ab/\nspu)$ parts. Note that for each $B_j$, with privacy budget at least $\ln(\ab/\nspu)$, we can guarantee
\[
 \expect{\totalvardist{\hat{\tp}_j}{\tp_j}} = O\Paren{\sqrt{\frac{\ab/\nspu}{(\ns/t) }}} = O\Paren{\sqrt{\frac{\ab}{\ns }}}.
\]
Since $\ab/\nspu<e^{\eps/2}$, each user can be assigned to $t'\eqdef \lfloor\frac{\priv}{2\ln (\ab/\nspu)}\rfloor$ blocks. Hence, the effective sample size increases by a factor of $t'$, and we have 
\[
 \expect{\totalvardist{\hat{\tp}_j}{\tp_j}}=O\left(\sqrt{\frac{\ab}{\ns t'}}\right)= O\left(\sqrt{\frac{\ab\ln(\ab/\nspu+1)}{\ns \priv}}\right).
\]
\end{enumerate}
}

\section{Connection to central DP and the shuffle model.} \label{app:shuffle}
In this section, we provide the proof of \cref{thm:shuffle}. The bound can be obtained by a combination of amplification by shuffling (\cref{thm:shuffling}) and the upper bound results in \cref{thm:rate_small_eps} and \cref{thm:small_m}. We assume without shuffling, each user sends an $\eps_0$-LDP message.

\paragraph{Small $\eps_0: \eps_0 \le 1$.} Note that in this case, in the shuffle model, $\eps = O\Paren{\eps_0 \sqrt{\frac{ \log(1/\delta)}{\ns}}}$. More specifically, for $\eps < \sqrt{\frac{9e \log(4/\delta)}{\ns}}$, there exists $\eps_0 =\eps \cdot  \sqrt{\frac{\ns}{9e \log(4/\delta)}} < 1$ such that the $\eps_0$-LDP algorithm is $(\eps, \delta)$-DP in the shuffle model. Plugging this into \cref{thm:rate_small_eps}, we get the desired bound in \cref{thm:shuffle}.

\paragraph{Large $\eps_0: 1 \le \eps_0 \le \log (\ab/\nspu)$.} In this case, in the shuffling model, $\eps = O\Paren{\sqrt{\frac{e^{\eps_0} \log(1/\delta)}{n}}}$. More specifically, when $\eps < \sqrt{\frac{\ab \log(1/\delta)^2}{\nspu \ns}}$, there exists $\eps_0 = \frac{1}{2} \log \frac{\ns\eps^2}{\log(1/\delta)} < \log(\ab/\nspu)$ such that the $\eps_0$-LDP algorithm is $(\eps, \delta)$-DP in the shuffle model. Plugging this into \cref{thm:small_m}, we get the desired bound in \cref{thm:shuffle}.

\section{Missing proofs for the lower bounds} \label{app:lower}
In this section, we present complete proofs for lower bound part of \cref{thm:rate_small_eps}, \cref{thm:small_m}, and \cref{thm:lb-large-m}. We use the information contraction framework in \cite{acharya2020unified} and the lower bound construction in \cite{ACLST:21}. 
Our hard instances are from the ``Paninski'' family~\cite{Paninski:08}.
Let $\gamma\in(0, 1/2)$ be a parameter related to the expected error. We consider a family of distributions defined as follows: for each vector $z\in\mathcal{Z}\eqdef\{-1, 1\}^{k/2}$, define a discrete distribution $\p_z$ as
\[
    \p_z(2i-1)=\frac{1+\gamma z_i}{k}, \quad \p_z(2i)=\frac{1-\gamma z_i}{k}.
\]

The $\nspu$ samples observed by each user can be viewed as a $\ab$-dimensional vector indicating the histogram from a multinomial distribution. We denote this distribution as $\pmul_z=\Multinomial(\nspu, \p_z)$. In this section, we use $\bm = (\bm(1), \bm(2), \ldots, \bm(k))$ to denote the histogram observed from $\pmul_z$ where $\bm(x)$ denotes the number of times $x$ appears in the observed $m$ samples.

When $\nspu$ is large ($\nspu \ge \ab$), we will consider the ``Poissonization" of the multinomial distribution, which we denote as $\ppoi_z=\Poisson(\nspu, \p_z)$. To generate a sample from $\ppoi_z$, first a random integer $M$ is generated from  $\poisson{\nspu}$ and the final observed samples are generated from $\Multinomial(M, \p_z)$. It is a folklore (\eg~\cite{gs009}) for $\bm = (\bm(1), \bm(2), \ldots, \bm(k)) \sim \ppoi_z$, we have: (1) All $\bm(x)$'s are mutually independent; (2) $\forall x \in [\ab],$ $\bm(x)$ follows a Poisson distribution with mean $\nspu \p_z(x)$. 

As discussed in \cref{sec:lower-bound}, we provide our proof in two separate regimes. In \cref{app:small_m}, we prove the lower bound part of \cref{thm:rate_small_eps} and  \cref{thm:small_m} by directly analyzing the multinomial setting. In \cref{app:large_m}, we prove \cref{thm:lb-large-m} using the Poissonization trick introduced above. We first introduce the information contraction framework in \cite{acharya2020unified} and necessary results.

\subsection{Information contraction bounds} \label{app:contraction}
Let $\cZ\eqdef\bool^\zdims$ and
$\{\q_z\}_{z\in\cZ}$ be a collection of
distributions over $\cX$, indexed by $z\in\cZ$.  For $z\in\cZ$,
denote by
$z^{\oplus i}\in\cZ$ the vector obtained by flipping the sign of the
$i$th coordinate of $z$. The following two assumptions on the density functions are needed.
\begin{assumption}
  \label{assn:decomposition-by-coordinates}
For every $z\in\cZ$ and $i\in[\zdims]$ it holds that $\q_{z^{\oplus i}} \ll \q_{z}$, and
there exist measurable functions $\phi_{z,i}\colon\cX\to\R$
such that
\[
 \frac{d\q_{z^{\oplus i}}}{d\q_z}=1+\alpha_{z,i}\phi_{z,i}
,
 \]
 where $|\alpha_{z,i}|\leq \alpha$ for some constant $\alpha\in \R$ independent of $z,i$. Moreover, for all $z\in\cZ$ and $i,j\in[\zdims]$,
$\bE{\q_{z}}{\phi_{z,i}\phi_{z,j}}=\indic{i=j}$. (In particular, $\shortexpect_{\p_z}[\phi_{z,i}^2]=1$.)
\end{assumption}

\begin{assumption}
  \label{assn:subgaussianity} There exists some $\sigma\geq 0$ such
that, for all $z\in\cZ$, the random vector $\phi_z(X)\eqdef
(\phi_{z,i}(X))_{i\in[\zdims]}\in\R^\zdims$ is $\sigma^2$-subgaussian
for $X\sim\q_z$, with independent coordinates.
\end{assumption}

Consider the following generating process. We first pick $Z$ uniformly at random from $\cZ$. Then each user observes a sample from $\q_Z$. The users follow the protocol $\Pi$ where each user uses a messaging scheme from a constrained set $\cW$ (\eg $\cW_\eps$ denotes all $\eps$-LDP schemes) to send a message $Y_i$ about there sample. The server observes all the messages $Y^\ns$ and estimate the distribution as $\hat{\p}$.

We denote the distribution of $Y^\ns$ when the samples are from $\q_Z$ as $\q_Z^{Y^\ns}$. We also denote the mixture of message distributions conditioned on a fixed $Z_i$ as the following $\q_{+i}^{Y^\ns}\eqdef \bEEC{\q_Z^{Y^\ns}}{Z_i=1}$, $\q_{-i}^{Y^\ns}\eqdef \bEEC{\q_Z^{Y^\ns}}{Z_i=1}$. Note that $\totalvardist{\q_{+i}^{Y^\ns}}{\q_{-i}^{Y^\ns}}$ can be viewed as an information measure that describes how much information $Y^\ns$ carries about $Z_i$. The following theorem provides an upper bound on this information measure.

\begin{theorem}[Main theorem of~\cite{acharya2020unified}]
  \label{thm:avg:coordinate} 
  Let $\Pi$ be a sequentially interactive
protocol using messaging schemes from $\cW$ and $(Y^\ns,U)$ be the transcript of $\Pi$ when
the input $X_1, \ldots, X_\ns$ is i.i.d.\ with common distribution
$\q_Z$.  Then,
under 
Assumption~\ref{assn:decomposition-by-coordinates},
we have
\begin{align}
\Paren{\frac{1}{\zdims}\sum_{i=1}^\zdims
	\totalvardist{\q_{+i}^{Y^\ns}}{\q_{-i}^{Y^\ns}}}^2\le 
	\frac{7}{\zdims} \ns \alpha^2 
	\max_{z\in\cZ}\max_{W\in\cW}\sum_{y\in\cY} 
	\frac{\var_{\q_{z}}[W(y\mid X)]}{\bE{\q_{z}}{W(y\mid 
	X)}},\label{eqn:var-bound}
\end{align}
Finally, if Assumption~\ref{assn:subgaussianity} holds as well, we have
\begin{align}
\Paren{\frac{1}{\zdims}\sum_{i=1}^\zdims\totalvardist{\q_{+i}^{Y^\ns}}{\q_{-i}^{Y^\ns}}}^2\le \frac{14 \ln 2}{\zdims}  \ns\alpha^2\sigma^2 \max_{z\in\cZ}\max_{W\in\cW}I(\q_z; W),\label{eqn:subgaussian-bound}
\end{align}
where $I(\q_z; W)$ denotes the mutual information $\mutualinfo{X}{Y}$ between the input $X\sim \q_z$
and the output $Y$ of the channel $W$ with $X$ as input.
\end{theorem}

In particular, it is proved in \cite{acharya2020unified} that when $\cW_\eps$ is the set of all $\eps$-LDP channels, we have for any $\q_z$,
\begin{equation}\label{eqn:ldp_info}
    \max_{z\in\cZ}\max_{W\in\cW_\eps}\sum_{y\in\cY} 
	\frac{\var_{\q_{z}}[W(y\mid X)]}{\bE{\q_{z}}{W(y\mid 
	X)}} \le \min \left\{4 \eps^2, e^\eps \right\}. 
\end{equation}

\subsection{$m\le k/e^\eps$ or $\eps<1$} \label{app:small_m}
We prove the minimax lower bound presented below. 
\begin{theorem}
\label{thm:lb_high_priv-small_m}
The minimax error rate satsifies
\[
\risk(\eps, \ab, \ns, \nspu)=\Omega\Paren{\sqrt{\frac{\ab}{\nspu\ns}}\lor \sqrt{\frac{k^2}{\nspu\ns(\priv^2\land e^\priv)}}}
\]
\end{theorem}

Note that when $\priv<1$, $\priv^2$ is the dominating term, leading to the tight lower bound in the high privacy regime (\cref{thm:rate_small_eps}). When $\priv\ge 1$, $e^\priv$ is the dominating term, which yields the desired lower bound for $\nspu<\ab/e^\priv$ (\cref{thm:lb-large-m}).

\begin{proof}


The first term is the lower bound in the centralized setting. We will mainly focus on the second term. Consider the same generating process described in \cref{app:contraction} with $\q_z = \pmul_z$. The following lemma shows that if $\Pi, \hat{\p}$ is a good estimator for $\p_Z$, we must be able to extract enough information about $Z$ from $Y^\ns$. The result follows from \cite[Lemma]{acharya2020unified}.
\begin{lemma} \label{lem:assouad-multi}
If $\Pi, \hat{\p}$ satisfies
\[
    \bEE{\totalvardist{\hp(\msgs)}{\p}} \le \frac{\gamma}{4},
\]
we must have
\begin{align}\label{eqn:assouad}
    \sum_{i = 1}^{\ab/2} \totalvardist{\p_{+i}^{{\rm mul}, Y^\ns}}{\p_{-i}^{{\rm mul}, Y^\ns}} = \Omega (\ab).
\end{align}
\end{lemma}

Next we upper bound the left hand side of \eqref{eqn:assouad} using \cref{thm:avg:coordinate}. In particular, we will prove $\pmul_z$ satisfies \cref{assn:decomposition-by-coordinates} with appropriate parameters.. 
\begin{lemma}\label{lem:alpha_multinomail}
$\{\pmul_z\}_{z\in\mathcal{Z}}$ satisfies \cref{assn:decomposition-by-coordinates} with $\alpha=O(\sqrt{\nspu\gamma^2/\ab})$ for $\gamma<\min\{1/2, \sqrt{\ab/(8\nspu+\ab)}\} $.
\end{lemma}
\begin{proof}
For a vector $\bm=(\bm_1, \ldots, \bm_{\ab})\in \mathbb{N}^k$, the probability mass is
\[
\pmul_z(\mathbf{m})=m!\prod_{i=1}^k\frac{\p_z(i)^{\bm_i}}{\bm_i!}.
\]
Therefore, 
\[
\frac{\pmul_{z^{\oplus i}}(\mathbf{m})}{\pmul_z(\mathbf{m})}=\left(\frac{1-\gamma z_i}{1+\gamma z_i}\right)^{\bm_{2i-1}}\left(\frac{1+\gamma z_i}{1-\gamma z_i}\right)^{\bm_{2i}}=\left(\frac{1+\gamma z_i}{1-\gamma z_i}\right)^{\bm_{2i}-\bm_{2i-1}}.
\]
We want to compute 
\[
\expectDistrOf{\pmul_z}{\left(\frac{\pmul_{z^{\oplus i}}(\bm)}{\pmul_z(\bm)}-1\right)^2}=\expectDistrOf{\pmul_z}{\left(\frac{\pmul_{z^{\oplus i}}(\bm)}{\pmul_z(\bm)}\right)^2}-1.
\]
First let $N=m_{2i-1}+m_{2i}$. For fixed $N$, $\bm_{2i}$ follows $\binomial{N}{p}$ where $p=(1-\gamma z_i)/2$. Hence we have
\begin{align*}
    \expectCond{\left(\frac{\pmul_{z^{\oplus i}}(\bm)}{\pmul_z(\bm)}\right)^2}{N}&=\expectDistrOf{\bm_{2i}\sim\binomial{N}{p} }{\left(\frac{1+\gamma z_i}{1-\gamma z_i}\right)^{4\bm_{2i}-2N }}\\
    &=\left(\frac{1+\gamma z_i}{1-\gamma z_i}\right)^{-2N}\left(p\left(\frac{1+\gamma z_i}{1-\gamma z_i}\right)^4+1-p\right)^N\\
    &=\left(\frac{1+\gamma z_i}{1-\gamma z_i}\right)^{-2N}\left(\frac{1+\gamma z_i}{2}\left(\left(\frac{1+\gamma z_i}{1-\gamma z_i}\right)^3+1\right)\right)^N\\
    &=\left(\frac{1}{2}\left(\frac{(1+\gamma z_i)^2}{1-\gamma z_i}+\frac{(1-\gamma z_i)^2}{1+\gamma z_i}\right)\right)^N\\
    &=\left(\frac{1+3\gamma^2}{1-\gamma^2}\right)^N.
\end{align*}
The second equality follows by the generating function of binomial distribution. Notice that $N\sim \binomial{\nspu}{2/\ab}$. Hence,
\begin{align*}
     \expect{\left(\frac{\pmul_{z^{\oplus i}}(\bm)}{\pmul_z(\bm)}-1\right)^2}&=\expect{\expectCond{\left(\frac{\pmul_{z^{\oplus i}}(\bm)}{\pmul_z(\bm)}\right)^2}{N}}-1\\
     &=\expectDistrOf{N\sim \binomial{\nspu}{2/\ab}}{\left(\frac{1+3\gamma^2}{1-\gamma^2}\right)^N}-1\\
     &=\left(\frac{2}{\ab}\frac{1+3\gamma^2}{1-\gamma^2}+1-\frac{2}{\ab}\right)^\nspu-1\\
     &=\left(1+\frac{8\gamma^2}{\ab(1-\gamma^2)}\right)^\nspu-1=:\alpha =O(m\gamma^2/k).
\end{align*}
Setting $\alpha_{z, i}=\sqrt{\expectDistrOf{\pmul_z}{\left(\frac{\pmul_{z^{\oplus i}}(\bm)}{\pmul_z(\bm)}-1\right)^2}}$ and $\phi_{z, i}=\left(\frac{\pmul_{z^{\oplus i}}(\bm)}{\pmul_z(\bm)}-1\right)/\alpha_{z_i}$ yields the desired result. It is obvious that $\expect{\phi_{z, i}\phi_{z, j}}=\indic{i=j}$. 
\end{proof}

Combining \cref{lem:alpha_multinomail}, \cref{thm:avg:coordinate}, and \cref{eqn:ldp_info}, we get:
\[
    \gamma = \Omega \Paren{\sqrt{\frac{\ab^2}{\nspu \ns \min\{\eps^2, e^\eps\}}}},
\]
completing the proof.
\end{proof}

\subsection{Large $\nspu$: $\nspu > \ab$.}
\label{app:large_m}
We prove \cref{thm:lb-large-m}, restated below.
\begin{theorem}
\label{thm:lb-large-m-app}
For $\ns>(k/\eps)^2$, $\nspu\ge \ab$, and $\eps>1$, the minimax error rate satisfies
\[
\risk(\eps, \ab, \ns, \nspu)=
\Omega\Paren{\sqrt{\frac{\ab}{\nspu\ns}}\lor \sqrt{\frac{k^2}{\nspu\ns\priv}}}.
\]
\end{theorem}


For $\nspu>\ab$, we prove the lower bound via Poissonization. Formally, define the following problems. 

\begin{description}
  \item[$\textsc{Multinomial}(\cW, \ns,\nspu)$:] each of the $\ns$ users obtains $\nspu$ samples from $\p$, and chooses a channel from $\cW$. The $\nspu n$ samples are i.i.d.
  \item[$\textsc{Poissonized}(\cW, \ns,\nspu)$:] For $1\leq t\leq \ns$, user $t$ observes $M_t$ samples from $\p$, where $(M_t)_{1\leq t\leq \ns}$ are independent $\poisson{\nspu}$, and chooses a channel from $\cW$. The $\sum_{t=1}^\ns M_t$ samples are i.i.d.
\end{description} 

We do not reduce $\textsc{Multinomial}(\cW_\eps, \ns,\nspu)$ to $\textsc{Poissonized}(\cW_\eps, \ns,\nspu)$ as~\cite[Lemma C.1]{ACLST:21} suggests. Instead, we consider the following channel. 
\begin{definition}
We define the family of channels `$\eps$-LDP+1bit', denoted as $\cW_{\eps, 1}$. A channel $W=W_1\otimes W_2 \in \cW_{\eps, 1}$ consists of two independent channels such that satisfies the following property given $X$, each user can send two messages $Y_1, Y_2$ through two independent channels $W_1$ and $W_2$: $Y_1\in\{0, 1\}$, and $Y_2$ satisfies LDP constraints.
\end{definition}

We have the following lemma:
\begin{lemma}
\label{lem:reduction}
If there exists a protocol that solves $\textsc{Multinomial}(\cW_\eps, \ns,\nspu)$ with accuracy $\gamma$, then there also exists a protocol that solves $\textsc{Poissonized}(\cW_{\eps, 1}, 20\ns,2\nspu)$ with accuracy $\gamma+e^{-2n/3}$. Moreover, the latter one is non-interactive if the former one is.
\end{lemma}
\begin{proof}
To design an algorithm that solves  $\textsc{Poissonized}(\cW_{\eps, 1}, 20\ns,2\nspu)$ with an algorithm for $\textsc{Multinomial}(\cW_\eps, \ns,\nspu)$, user $u$ first sends a bit $Y_{u, 1}$ indicating whether it receives more than $\nspu$ samples. Then, if the user has more than $\nspu$ samples, then it keeps only $\nspu$ samples and sends a message $Y_{u, 2}$ according to the $\eps$-LDP protocol for $\textsc{Multinomial}(\cW_\eps, \ns,\nspu)$. Otherwise, duplicate the existing samples so that the user has $\nspu$ samples, and also send $Y_{u, 2}$ according to the $\eps$-LDP protocol. $Y_{u, 2}$ obviously satisfies $\eps$-LDP constraints. Hence $Y_u=(Y_{u, 1}, Y_{u, 2})$ is a valid message from a channel in $\cW_{\priv, 1}$.

The server keeps the messages such that $Y_{u, 1}=1$, and use the corresponding $Y_{u, 2}$ to estimate the underlying distribution.

To bound the accuracy of the above protocol, first note that for $M\sim \poisson{2\nspu}$, we have
\[
\probaOf{M<\expect{M}/2=\nspu}\le e^{-\nspu/6}\le e^{-1/6}.
\]
Therefore, each user receives at least $\nspu$ samples with probability at least $1-e^{-1/6}>3/20$. Using Chernoff bound, with probability at least $1-\delta \eqdef 1- e^{2\ns/3}$, at least $\ns$ users has at least $\nspu$ samples. Hence the expected error is at most
\[
\gamma(1-\delta)+\delta\le \gamma + e^{-2\ns/3}.
\]
\end{proof}

Next we focus on the Poisonized setting. Similar to \cref{lem:assouad-poi}, we can obtain the following lemma.
\begin{lemma} \label{lem:assouad-poi}
  Under the Poissonized sampling model, if $\Pi, \hat{\p}$ satisfies
\[
    \bEE{\totalvardist{\hp(\msgs)}{\p}} \le \frac{\gamma}{4},
\]
we must have
\begin{align}\label{eqn:assouad_poi}
    \sum_{i = 1}^{\ab/2} \totalvardist{\p_{+i}^{{\rm poi}, Y^\ns}}{\p_{-i}^{{\rm poi}, Y^\ns}} = \Omega (\ab).
\end{align}
\end{lemma}

Following the proof of \cite[Theorem C.7, C.10]{ACLST:21}, we can obtain the following upper bound on the obtained information for the Poissonized problem under $\cW_{\eps, 1}$.

\begin{lemma}\label{lem:sum_informaiton-poi}
  For any interactive protocol with channels from $\cW_{\eps, 1}$, when $\nspu > \ab \log \ab$, we have there exists a constant $C$ such that
  \[
    \sum_{i = 1}^{\ab/2} \totalvardist{\p_{+i}^{{\rm poi}, Y^\ns}}{\p_{-i}^{{\rm poi}, Y^\ns}} \le C \cdot \ns \frac{\gamma^2 \nspu}{
    \ab} \cdot \Paren{\nspu \gamma^2 + \max_{z \in \cZ}\max_{W \in \cW_{\eps,1} }I(\ppoi_z; W)}.
  \]
\end{lemma}

The final ingredient is to prove a mutual information bound for $\cW_{\eps, 1}$ to apply~\cite[Theorem 2]{ACLST:21}
\begin{lemma}
\label{lem:mutual_info}
The mutual information $\max_{z \in \cZ}\max_{W \in \cW_{\eps,1}} I(\ppoi_z; W)\le \eps\log_2e+1$.
\end{lemma}
\begin{proof}
Let $X\sim \ppoi_z$ and $Y=(Y_1, Y_2)$ be a message sent through a channel in $\cW_{\eps, 1}$.
\begin{align*}
    I(Y_1, Y_2; X)&=\expectDistrOf{X}{\kldiv{p_{Y|X}}{p_Y} }\\
    &=\expectDistrOf{X}{\sum_{y}p_{Y|X}(y)\log \frac{p_{Y|X}(y)}{p_Y(y)}}\\
    &=\expectDistrOf{X}{\sum_{y_1}W_1(y_1|X)\sum_{y_2}W_2(y_2|X)\left(\log\frac{W(y_2|X)}{p_Y(y_2|y_1)}+\log\frac{W(y_1|X)}{p_Y(y_1)}\right)}\\
    &=\expectDistrOf{X}{\sum_{y_1}W_1(y_1|X)\sum_{y_2}W_2(y_2|X)\log\frac{W(y_2|X)}{p_Y(y_2|y_1)}+\kldiv{p_{Y_1|X}}{p_{Y_1}}}\\
    &\le \eps\log e+I(Y_1; X)\\
    &\le \eps\log e + 1
\end{align*}
The second to last inequality is due to LDP constraint on $Y_2$. The final inequality is due to $I(Y_1; X)\le H(p^{W_1})$ where $p^{W_1}=\expectDistrOf{p}{W_1(Y_1|X)}$. Since $Y_1\in \{0, 1\}$, the entropy must be at most 1.
\end{proof}

Combining \cref{lem:assouad-poi}, \cref{lem:sum_informaiton-poi}, and \cref{lem:mutual_info}, we have
\[
\ns \frac{\nspu \perturb^2}{\ab} (\priv + 1 + \nspu \perturb^2 ))  =
	\Omega(\ab),
\]
which implies
\[
\perturb = \Omega\Paren{\min \mleft\{\sqrt{\frac{\ab^2}{\nspu \ns \priv}}, 
	\sqrt{\frac{\ab}{\nspu \sqrt{\ns}}} \mright\}} = \Omega\Paren{\sqrt{\frac{\ab^2}{\nspu \ns \priv}}}\,
\]

The final equality is due to $\ns>(\ab/\priv)^2$. By~\cref{lem:reduction} the same bound holds for $\textsc{Multinomial}(\cW_\eps, \ns,\nspu)$ up to constant factors.

\ignore{
Following the proof of \cite[Theorem C.7, C.10]{ACLST:21}, we can obtain the following upper bound on  Poissonized problem for $\cW_{\eps, 1}$.

\begin{lemma}
  \label{lemma:family:pz}
  The family $\{\ppoi_z\}_{z\in\bool^\zdims}$ of probability distributions over 
  $\N^\zdims$ satisfies 
  Assumption~\ref{assn:decomposition-by-coordinates} with
  \[
      \alpha \eqdef \sqrt{e^{8\nspu\perturb^2/\ab}-1} = \bigO{\sqrt{\nspu\perturb^2/\ab}},
  \]  
  where the asymptotics are as $\gamma\to 0$; and, for $i\in[\zdims]$ and $z\in\bool^\zdims$,
  \[
      \alpha_{z,i} \eqdef \sqrt{e^{\frac{4\nspu\perturb^2}{\ab(1+\gamma z_i)}}-1},\qquad \phi_{z,i} (\mathbf{\nsamps}) \eqdef  \frac{1}{\alpha_{z,i}} 
	\Paren{\Paren{\frac{1-\perturb z_i}{1+\perturb 
				z_i}}^{\mathbf{m}_i} e^{\frac{2 \nspu \perturb z_i}{\ab}} - 
	1}
  \]  
\end{lemma} 
Let $\theta_z(i)=\frac{1+\gamma z_i}{\ab}$. As in \cite[Section C.4]{ACLST:21}, $\phi_{z, i}(\mathbf{\nsamps})$ can be decomposed into three parts $\phi_{z, i}(\mathbf{\nsamps})=\zeta_{z, i}(\mathbf{m})+\xi_{z, i}(\mathbf{\nsamps})+\psi_{z, i}(\mathbf{\nsamps})$ where
\begin{align}
	\zeta_{z, i}(\mathbf{\nsamps}) 
	    &= \frac{-2 \perturb z_i}{\alpha_{z,i}}  \Paren{\mathbf{m}_i - \frac{\nspu(1+\perturb z_i)}{\ab}}\indic{\mathbf{m}_i \le 10\theta_z(i)} \label{eqn:subg} \\
	\xi_{z, i}(\mathbf{\nsamps}) 
	    &= \frac{-2 \perturb z_i}{\alpha_{z,i}}  	\Paren{\mathbf{m}_i 	- \frac{\nspu(1+\perturb z_i)}{\ab}}\indic{\mathbf{m}_i > 10\theta_z(i)} \label{eqn:res2} \\
	    \psi_{z, i}(\mathbf{\nsamps})&=\phi_{z, i}(\mathbf{\nsamps})-\zeta_{z, i}(\mathbf{\nsamps}) -\xi_{z, i}(\mathbf{\nsamps})
\end{align}

To obtain the desired lower bound, we need to bound the following term.
\begin{align}
	 &\sum_{i=1}^\zdims \sum_{y\in\cY} 
	\frac{\bE{\ppoi_{z}}{\phi_{z,i}(X)W(y\mid X)}^2}{\bE{\ppoi_{z}}{W(y\mid X)}} \notag\\
	&\le  
	3\sum_{i=1}^\zdims \sum_{y\in\cY} 
	\frac{\bE{\ppoi_{z}}{\zeta_{z,i}(X)W(y\mid X)}^2}{\bE{\ppoi_{z}}{W(y\mid X)}} +  
	3\sum_{i=1}^\zdims \sum_{y\in\cY} 
	\frac{\bE{\ppoi_{z}}{\xi_{z,i}(X)W(y\mid X)}^2}{\bE{\ppoi_{z}}{W(y\mid X)}} \notag\\
  &\quad +  
	3\sum_{i=1}^\zdims \sum_{y\in\cY} 
	\frac{\bE{\ppoi_{z}}{\psi_{z,i}(X)W(y\mid X)}^2}{\bE{\ppoi_{z}}{W(y\mid X)}}. \label{eq:bounding:3terms}
\end{align}

We bound three terms separately. Via the moment generating function of Poisson distributions and ~\cite[Claim C.9]{ACLST:21}, using the same argument as in \cite{ACLST:21},we have 
  \[
  \sum_{i=1}^\zdims \sum_{y\in\cY} 
	\frac{\bE{\ppoi_{z}}{\xi_{z,i}(X)W(y\mid X)}^2}{\bE{\ppoi_{z}}{W(y\mid X)}}\le 
	8\ab e^{-2\frac{\nspu}{\ab}},\quad \sum_{i=1}^\zdims \sum_{y\in\cY} 
\frac{\bE{\ppoi_{z}}{\psi_{z,i}(X)W(y\mid X)}^2}{\bE{\ppoi_{z}}{W(y\mid X)}} 
\lesssim 
2\nspu\perturb^2
  \]
By~\cite[Claim C.9]{ACLST:21}, $\zeta_{z, i}$ is subgaussian, and hence by \cref{lem:mutual_info} and \cref{eqn:subg},
\[
\sum_{i=1}^\zdims \sum_{y\in\cY} 
	\frac{\bE{\ppoi_{z}}{\zeta_{z,i}(X)W(y\mid X)}^2}{\bE{\ppoi_{z}}{W(y\mid X)}} \le 
	O(\priv).
\]
}

\section{Additional experiment results}
\label{app:experiment}
\subsection{Interactive algorithm}
In this section, we present additional experiment results for our interactive algorithms. 
\paragraph{High privacy regime $\eps\le 1$}
We show an additional result with larger alphabet size ($\ab=100$). We can see that our algorithm outperforms 1-sample HR by a large margin, and the error is always within a constant factor of all-sample HR.
\begin{figure}[h]
    \centering
    \includegraphics[width=0.4\linewidth]{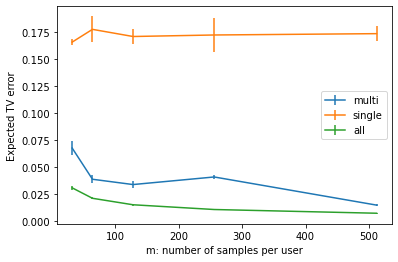}
    \includegraphics[width=0.4\linewidth]{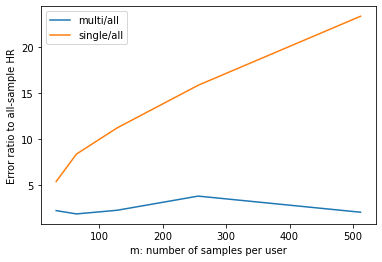}
    \caption{Performance in the high privacy regime with $\priv=0.9$, $\ab=100$, $\p$ uniform. \textbf{Left}: expected error of our algorithm (blue), 1-sample HR (orange), and all-sample HR (green). \textbf{Right}: orange/green and blue/green ratio in the left plot.}
    \label{fig:high_priv_k100}
\end{figure}

\subsection{Non-interactive algorithm}
In this section we present experiment results for the non-interactive algorithm. We mainly focus on the case when $\ab=2$  (binomial estimation) and the high privacy regime ($\eps = O(1)$) as this is the only part where interactivity is needed in the interactive version of the algorithm for all other regimes. Thus, it is sufficient to demonstrate the difference between the two versions when $\ab=2, \priv=O(1)$ since we can substitute this part in other regimes to make them non-interactive as well.

We make some minor changes in our implementation,
\begin{enumerate}
    \item We choose $\CI=0.6$ and $\CR=2.1$, much smaller than the constants used in our proofs.
    \item We divide users into 3 groups instead of 4 with $|S_1|=|S_2|=|S_3|=\ns/3$, dropping the users that send $\indic{Z_u\ge 1}$. Users in $S_1$ are used for the localization stage. Users $S_2$ and $S_3$ are used in the refinement stage to obtain empirical estimates of $R_2$ and $R_3$.
\end{enumerate}

We compare the non-interactive version with the interactive algorithm and the baselines (1-sample HR and all-sample HR ). 
The results are shown in~\cref{fig:non-interactive_k2}. 
\begin{figure}[h]
    \centering
    \begin{tabular}{c c c c }
    \includegraphics[width=0.23\linewidth]{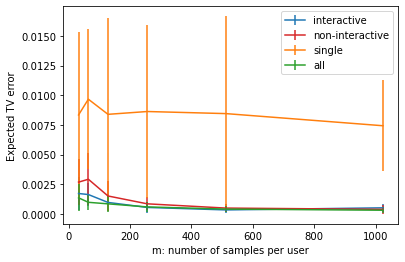}
    &\includegraphics[width=0.23\linewidth]{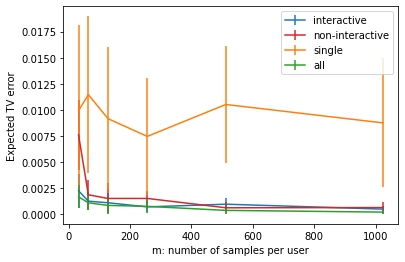}
    &\includegraphics[width=0.23\linewidth]{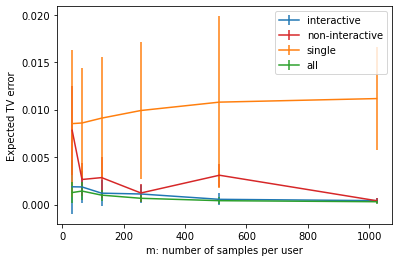}
    &\includegraphics[width=0.23\linewidth]{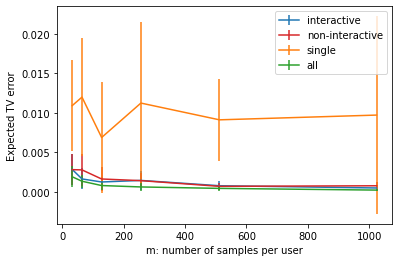}\\
    \includegraphics[width=0.23\linewidth]{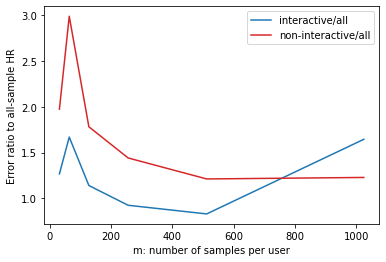}
    &\includegraphics[width=0.23\linewidth]{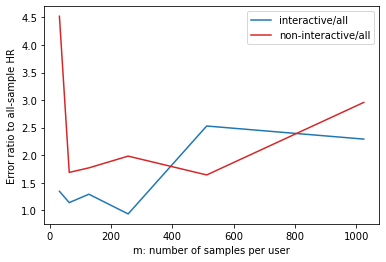}
    &\includegraphics[width=0.23\linewidth]{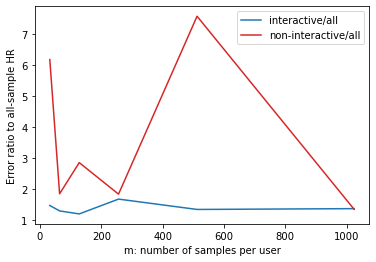}
    &\includegraphics[width=0.23\linewidth]{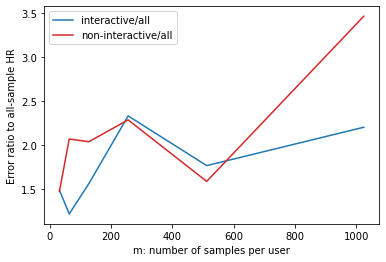}\\
    $p=0.1$ &$p=0.2$ & $p=0.3$ &$p=0.4$
    \end{tabular}
    \caption{Performance of the non-interactive algorithm in the high privacy regime with $k=2$, $\eps=0.9, \ns=9000$. $\nspu=[32, 64, 128, 256, 512, 1024]$. \textbf{Top row}: the expected TV error of non-interactive (orange), interactive (red), 1-sample HR (orange) and all-sample HR (green) with respect to $\nspu$. Mean and std are reported over 20 independent runs. \textbf{Bottom row}: the ratio to all-sample HR of non-interactive (red) and interactive (blue) algorithms w.r.t. $\nspu$. }
    \label{fig:non-interactive_k2}
\end{figure}

\cref{fig:non-interactive_k2} shows that the non-interactive algorithm significantly outperforms the 1-sample HR baseline, and the performance is reasonably close to the interactive version and all-sample HR. This demonstrates the possibility of implementing a non-interactive algorithm that improves with increasing $\nspu$ and matches our theoretical bounds. However, we do observe that the non-interactive algorithm is less stable and usually performs worse than the interactive one. We view our work mainly as a theoretical investigation of the role of multiple samples in user-level LDP, and the experiments are mainly used to demonstrate algorithmic ideas. We leave optimizing the constants and implementation details to make the algorithm more stable as future work.

\end{document}